\documentclass[twoside,letterpaper]{article}

\usepackage[preprint]{aistats2027}
\usepackage[round,authoryear]{natbib}

\ifdefined\pdfpageheight
\fi

\usepackage[utf8]{inputenc} 
\usepackage[T1]{fontenc}    
\usepackage{hyperref}       
\hypersetup{
    colorlinks = true,
    linkcolor  = blue!70!black,
    citecolor  = cyan!99!black,
    urlcolor   = cyan!70!black
}
\usepackage{url}            
\usepackage{booktabs}       
\usepackage{amsfonts}       
\usepackage{nicefrac}       
\usepackage{microtype}      
\usepackage{xcolor}         

\usepackage{amsmath, amssymb, amsthm}
\usepackage{wrapfig}
\usepackage{graphicx}   
\usepackage{caption}    
\usepackage{subcaption} 
\usepackage{float}      
\usepackage{placeins}   

\newtheorem{theorem}{Theorem}[section]
\newtheorem{lemma}[theorem]{Lemma}
\newtheorem{definition}[theorem]{Definition}
\newtheorem{proposition}[theorem]{Proposition}
\newtheorem{corollary}[theorem]{Corollary}

\theoremstyle{remark}
\newtheorem{remark}[theorem]{Remark}

\newenvironment{restated}[1][]{\par\medskip\noindent\textbf{#1 (restated).}\itshape}{\par\medskip}

\definecolor{revision}{rgb}{0,0,0}
\hypersetup{pdftitle={On the Approximation Complexity of Matrix Product Operator Born Machines},pdfauthor={Chao Li, Zerui Tao, Yuchen Cong, Jian Xu, Qibin Zhao}}
\begin{document}
\runningtitle{On the Approximation Complexity of MPO Born Machines}
\runningauthor{Li, Tao, Cong, Xu, Zhao}
\twocolumn[
\aistatstitle{On the Approximation Complexity of\\
Matrix Product Operator (MPO) Born Machines}
\aistatsauthor{Chao Li{\normalfont\textsuperscript{*}} \And Zerui Tao \And Yuchen Cong \And Jian Xu \And Qibin Zhao}
\aistatsaddress{RIKEN-AIP \And RIKEN-AIP \And Jutentedo University \And RIKEN iTHEMS \And RIKEN-AIP}
]
\begingroup
\renewcommand{\thefootnote}{*}
\footnotetext[1]{Email: \href{mailto:chao.li@riken.jp}{chao.li@riken.jp}.}
\endgroup

\begin{abstract}
Matrix product operator Born machines (MPO-BMs) are tractable tensor-network models for probabilistic modeling, yet their approximation capabilities remain poorly understood in theory.
In this paper, we characterize these capabilities from both negative and positive perspectives.
As a negative result, we prove that constructing a polynomial-bond MPO-BM approximation to an arbitrary continuous target with bounded KL error is NP-hard.
On the positive side, for score-based variational inference, we prove that targets satisfying a log-Sobolev inequality admit polynomial-bond MPO-BM approximations with provable KL guarantees when the loss-induced Hamiltonian is local and has a constant spectral gap.
Furthermore, under the same conditions, we construct local importance-sampling estimators and prove that polynomially many local score queries suffice to preserve the KL guarantees.
\end{abstract}


\section{Introduction}
Tensor networks provide remarkably efficient parameterizations for extremely high-dimensional objects.
In particular, matrix product states (MPS, \citealp{perez2006matrix}), also known as tensor trains~\citep{oseledets2011tensor} and their operator counterpart, matrix product operators (MPOs), can represent exponentially large vectors or operators using a number of parameters that grows only \emph{linearly} with the dimension.
This efficiency has led to a wide range of applications in machine learning~\citep{novikov2015tensorizing,kossaifi2020tensor,sozykin2022ttopt,basharin2024faster,chen2024quanta,qiu2025steps}.

In these applications, one important branch is probabilistic modeling, where tensor networks are used to represent probability distributions.
For instance, MPS Born machines~\citep{han2018unsupervised} have been proposed for learning discrete distributions from images and physical systems.
Subsequent works extend this idea to more general tensor network architectures~\citep{cheng2019tree,liu2023tensor}, basis functions~\citep{meiburg2025generative}, and even quantum implementations~\citep{tian2023recent}.
Among these models, MPO Born machines (MPO-BMs) are particularly appealing:
they inherit the computational efficiency of tensor networks while exhibiting strictly greater expressive power than many probabilistic modeling methods~\citep{glasser2019expressive,loconte2025sum,martires2025quantum}.

Despite these advantages, a systematic understanding of their approximation properties remains elusive.
In particular, while MPO-BMs are known to be expressive, it remains unclear when such expressivity comes with provable approximation guarantees.
Given a target distribution, several fundamental questions are:
(i)  Can MPO-BMs approximate arbitrary distributions efficiently?
(ii) Which classes of distributions admit accurate approximation with low bond dimension?
(iii) How many queries to the target are sufficient to guarantee such an approximation?
Answering these questions is key to understanding the fundamental limits of MPO-BMs and, more broadly, tensor-network probabilistic models, as well as to guiding the design of principled learning algorithms.

In this work, we provide a rigorously-proved answer to the three questions by characterizing both the limitations and the efficient regimes of MPO-BMs.
We begin with a negative result: we show that finding a compact MPO-BM approximation with small KL error is NP-hard.
This implies that MPO-BMs cannot serve as universally efficient approximators in the worst case, even when approximation (rather than exact representation) is considered.

We then turn to the positive side by analyzing score-based variational inference (SBVI,~\citealp{cai2024eigenvi},~\citealp{cong2026quanvi}).
For a broad class of structured distributions, we establish polynomial upper bounds on both the MPO bond dimension and the number of local score queries under explicit structural and regularity assumptions.
To achieve the goal, we connect SBVI to the well-known ground-state problem in physics, under which we 
show that the target distribution admits a MPO approximation with polynomial bond dimension and provable KL guarantees when the loss-induced Hamiltonian exhibits locality and a constant spectral gap.
These conditions can be grounded in properties of the target distribution: path-graph Markov random fields naturally induce Hamiltonian locality, while a log-Sobolev inequality provides sufficient conditions for a constant spectral gap.

Furthermore, under the stated structural assumptions, we establish polynomial score-query complexity for achieving the same KL approximation guarantees through local importance sampling.
For comparison, Existing approaches based on explicit tensor-product expansions, such as EigenVI~\citep{cai2024eigenvi}, incur computational costs that grow exponentially with dimension.

We summarize our main contributions as follows:
\begin{itemize}
    \item We prove that constructing a polynomial-bond MPO-BM approximation with small constant KL error is NP-hard in the continuous setting.

    \item We establish polynomial-bond MPO-BM approximations with provable KL guarantees under locality, spectral-gap, and log-Sobolev assumptions, and relate these Hamiltonian conditions to structural and regularity properties of the target distribution.

    \item We prove that, under the same assumptions, polynomially many local score queries suffice to preserve the KL guarantee via local importance sampling.
\end{itemize}

We complement our theoretical results with numerical experiments in Appendix~\ref{sec:app_experimental_details}.

\section{Preliminaries}
We briefly review MPOs, MPO-BMs, and the SBVI formulation used in our analysis.

\paragraph{Notation.}
We write $[D]:=\{1,\ldots,D\}$.
For $S\subseteq[D]$, $x_S$ and $x_{\bar S}$ denote the corresponding subvectors.
For a matrix $A$, $\|A\|_{S^p}$ denotes the Schatten-$p$ norm; in particular,
$\|A\|_{S^\infty}$ and $\|A\|_{S^1}$ correspond to the spectral norm and trace norm, respectively.
We write $A \succeq 0$ and $A \succ 0$ to denote that $A$ is positive semidefinite (PSD) and positive definite (PD), respectively, and $\operatorname{tr}(\cdot)$ for the trace.
For probability densities $p$ and $q$, we write $q\ll p$ if $q$ is absolutely
continuous with respect to $p$.
For a density $p$, $H^1(p)$ denotes the weighted Sobolev space of functions
$f$ satisfying
$\int (f^2+\|\nabla f\|_2^2)p\,dx<\infty$.
The notation $\mathrm{poly}(\cdot)$ denotes a polynomial in its arguments;
$\bigotimes$ denotes the tensor product; and $C^1$ denotes the class of
continuously differentiable functions.

\subsection{Matrix Product Operator}
We begin by introducing the matrix product operator (MPO), a tensor network representation for high-dimensional linear operators associated with a path graph.

Let $\mathbb{H}:=\bigotimes_{d=1}^{D}\mathbb{R}^K$ be a $D$-site Hilbert space with local dimension $K$.
A linear operator $\rho:\mathbb{H}\rightarrow\mathbb{H}$ can be viewed as a tensor of order $2D$ with entries
$\rho(i_1,\ldots,i_D, j_1,\ldots,j_D)$, where $i_d,j_d\in[K]$ for all $d\in[D]$.

\begin{definition}[\textbf{Matrix Product Operator}]
An operator $\rho$ is said to admit an MPO representation with bond dimensions $\{R_d\}_{d=0}^D$ (with $R_0=R_D=1$) if its entries can be written as
\begin{equation}
\begin{aligned}
&\rho(i_1,\ldots,i_D, j_1,\ldots,j_D)\\
&\quad=
\sum_{\alpha_1,\ldots,\alpha_{D-1}}
G^{(1)}_{1,\,i_1,j_1,\,\alpha_1}
G^{(2)}_{\alpha_1,\,i_2,j_2,\,\alpha_2}\cdots
G^{(D)}_{\alpha_{D-1},\,i_D,j_D,\,1},
\end{aligned}
\end{equation}
where $G^{(d)} \in \mathbb{R}^{R_{d-1} \times K \times K \times R_d}$ are the core tensors for $d\in[D]$.
\end{definition}

We remark that the maximum bond dimension $R:=\max_d R_d$ controls the expressive power of the MPO.
The number of parameters scales as $\mathcal{O}(DK^2R^2)$, which is polynomial in $D$ when $R=\mathrm{poly}(D)$.
This representation thus provides a compact parameterization of exponentially large operators.

For convenience, we denote by $\mathrm{MPO}(R)$ the set of operators that admit an MPO representation with maximum bond dimension at most $R$.

\subsection{MPO Born Machine}\label{sec:MPO_Born_Machine}
In this paper, we consider Born machine~\citep{han2018unsupervised} defined via density operators and feature maps.
Given a $D$-dimensional random variable $x\in\mathbb{R}^{D}$, the model defines a density function as follows:
\begin{equation}
    q_\rho(x)=\Phi(x)^\top \rho \,\Phi(x),\label{eq:born_machine}
\end{equation}
where $\Phi(x) \in \mathbb{R}^{K^D}$ denotes a tensorized feature map and $\rho \in \mathbb{R}^{K^D \times K^D}$ indicates the only learnable parameter.

\paragraph{Feature map.}
The feature map is defined as a tensor product of local basis functions:
\begin{equation}
    \Phi(x) = \bigotimes_{d=1}^D \phi(x_d),
\end{equation}
where $\phi: \mathbb{R} \to \mathbb{R}^K$ is a vector-valued function satisfying the orthonormality condition $\int{}\phi(x)\phi(x)^\top{}dx=I$.

\paragraph{Density operator.}
The parameter $\rho$ is assumed to be a \emph{density operator}~\citep{nielsen2001quantum}, which must be PSD with unit trace:
\begin{equation}
    \rho \succeq 0, \quad \mathrm{tr}(\rho) = 1.\label{eq:density_operator}
\end{equation}
Note that, under the assumption of~\eqref{eq:density_operator} and the orthonormality of the feature map, $q_\rho(x)$ clearly defines a valid probability density function.

\paragraph{MPO-BM.}
In this paper, we are interested in $\rho$ in the MPO class, i.e., $\rho \in \mathrm{MPO}(R)$\footnote{Keep in mind that $\rho$ always satisfies the assumption of~\eqref{eq:density_operator} as a density operator, even if it is not clarified explicitly in the rest of the paper.}
with bond dimension $R=\mathrm{poly}(D)$.
Since a general density operator $\rho$ requires $\mathcal{O}(K^{2D})$ parameters, the MPO restriction reduces the parameter complexity to $\mathcal{O}(D K^2 R^2)$, which is polynomial in $D$.
This restriction is essential for computational efficiency and forms the basis of our analysis.

An important property of MPO-BMs is that they admit
\emph{tractable marginalization} for coordinate-wise interval queries.
For any subset $S\subseteq[D]$ and intervals $I_d$ with rational
endpoints, define the coordinate cylinder
$
C=\{x\in\mathbb{R}^D:x_d\in I_d \text{ for all } d\in S\}.
$
Its probability is
\begin{equation}
Q_\rho(C)
=\int_C q_\rho(x)\,dx
=\operatorname{tr}\!\left(
\rho\bigotimes_{d=1}^D B_d
\right),
\end{equation}
where
$
B_d=
\displaystyle\int_{I_d}\phi(t)\phi(t)^\top\,dt
$ for $d\in S$; and an identity matrix otherwise.

Marginalization preserves the MPO chain structure: each core is contracted with the corresponding local matrix $B_d$, which equals $I_K$ for an unconstrained coordinate by orthonormality.
Given these local matrices, the contraction requires
$\mathcal{O}(DK^2R^2)$ arithmetic operations, which is polynomial in $D$ for fixed $K$ and $R=\mathrm{poly}(D)$.
We assume that the required one-dimensional basis integrals can be evaluated efficiently to the prescribed numerical accuracy.



\subsection{Score-Based Variational Inference and EigenVI}\label{sec:SBVI_EigenVI}
To analyze both model and query complexity of MPO-BMs, we consider the task of \emph{score-based variational inference} (SBVI), motivated by the big success of score-based methods in generative modeling.
In our context, given a target density $p(x)$, SBVI seeks a model $q_\rho$ by minimizing the Fisher divergence:
\begin{equation}
    \begin{split}
        \min_\rho\frac{1}{2}\underbrace{\int{}q_\rho(x)\left\Vert
        \nabla_x\log{}p(x){}-\nabla_x\log{}q_\rho(x)
        \right\Vert_F^2dx}_{\operatorname{D_F}(q\Vert{}p):=}\label{eq:sbvi_loss}
    \end{split}.
\end{equation}
Here $D_F(q\Vert p)$ represents the Fisher divergence.
The following lemma shows that the Fisher divergence controls the Kullback--Leibler (KL) divergence if $p$ satisfies a \emph{log-Sobolev inequality}~\citep{gross1975logarithmic}
with constant $c$, written as $LSI(c)$ for short
(see Definition~\ref{def:LSI}).
\begin{lemma}
\label{thm:upper_bound_of_Fisher_div_LSI}
Let $q\ll p$ and assume $\sqrt{q/p}\in H^1(p)$. If $p$ satisfies
$LSI(c)$ for some $c>0$, then
\begin{equation}
D_{\mathrm{KL}}(q\Vert p)\le \frac{1}{2c} D_F(q\Vert p),
\end{equation}
where
$
D_{\mathrm{KL}}(q \Vert p)
= \int q(x)\bigl(\log q(x)-\log p(x)\bigr)\,dx
$ denotes the KL divergence.
\end{lemma}


\begin{remark}[\textbf{Examples of LSI constants}]
For a Gaussian target $p=\mathcal{N}(\mu,\Sigma)$ with
$\Sigma\succ0$, the optimal LSI constant is
$c=\lambda_{\min}(\Sigma^{-1})$, where $\lambda_{\min}$ denotes
the smallest eigenvalue~\citep{gross1975logarithmic}.
More generally, if $p\propto e^{-V}$ with $V\in C^2$ and
$\nabla^2 V(x)\succeq cI$ for all $x$, the Bakry--\'{E}mery
criterion guarantees $LSI(c)$~\citep{bakry1985diffusions}.
In our experiments, we consider Gaussian and non-Gaussian $\phi^4$-type
targets with quartic potentials~\citep{de2021scaling}, for which
the LSI constants can be established explicitly; see
Appendix~\ref{sec:app_experimental_details}.
\end{remark}

\paragraph{EigenVI formulation.}
To solve~\eqref{eq:sbvi_loss}, \citet{cai2024eigenvi} parameterize
$\rho=\theta\theta^\top$ and
$q_\theta(x)=|\theta^\top\Phi(x)|^2$, where
$\theta\in\mathbb{R}^{K^D}$ has unit norm.  
Under this parameterization,~\eqref{eq:sbvi_loss} reduces to the eigenvalue problem
\begin{equation}
    \min_{\theta}\theta^\top{}H\theta,\quad{}s.t.\,\Vert\theta\Vert_2^2=1,\label{eq:eigenvi_eigenform}
\end{equation}
where the matrix $H$ admits the decomposition
\begin{equation}
    \begin{split}
        &H:=\int{}H(x)dx\\
        &\quad=\int{}H_1(x)dx-2\int{}H_2(x)dx+4\int{}H_3(x)dx,\\
        &H_1(x)=\bigotimes_{d=1}^D\left(\phi(x_d)\phi(x_d)^\top\right)\Vert{}s(x)\Vert_2^2,\\
        &H_2(x)=\Phi(x)s(x)^\top\dot{\Phi}(x)+\dot{\Phi}(x)^\top{}s(x)\Phi(x)^\top,\\
        &H_3(x)=\dot{\Phi}(x)^\top\dot{\Phi}(x),
    \end{split}\label{eq:def_H}
\end{equation}
with $s(x):=\nabla_x \log p(x)$ the score function.
The Jacobian $\dot{\Phi}(x)\in\mathbb{R}^{D\times K^D}$ is given by
\begin{equation}
\label{eq:H}
\begin{aligned}
&\dot{\Phi}(x) = {}\\[-0.5ex]
&\quad
\begin{pmatrix}
\dot{\phi}(x_1)\otimes
\Bigl[\bigotimes\limits_{d=2}^{D}\phi(x_d)^\top\Bigr]
\\[1ex]
\phi(x_1)^\top\otimes\dot{\phi}(x_2)\otimes
\Bigl[\bigotimes\limits_{d=3}^{D}\phi(x_d)^\top\Bigr]
\\[0.5ex]
\vdots
\\[0.5ex]
\Bigl[\bigotimes\limits_{d=1}^{l-1}\phi(x_d)^\top\Bigr]
\otimes\dot{\phi}(x_l)\otimes
\Bigl[\bigotimes\limits_{d=l+1}^{D}\phi(x_d)^\top\Bigr]
\\[0.5ex]
\vdots
\\[0.5ex]
\Bigl[\bigotimes\limits_{d=1}^{D-1}\phi(x_d)^\top\Bigr]
\otimes\dot{\phi}(x_D)
\end{pmatrix},
\end{aligned}
\end{equation}

where $\dot{\phi}(x):=\left[\frac{d\phi_1(x)}{d{}x},\frac{d\phi_2(x)}{d{}x},\ldots,\frac{d\phi_K(x)}{d{}x}\right]\in\mathbb{R}^{1\times{}K},\,d\in[D]$.
The construction satisfies
$\theta^\top H\theta=D_F(q_\theta\Vert p)$, which implies
$H\succeq0$.
By the Rayleigh--Ritz principle, the minimum
in~\eqref{eq:eigenvi_eigenform} is $\lambda_{\min}(H)$ and is achieved
by any corresponding unit eigenvector.
Note that the factor $1/2$ in~\eqref{eq:sbvi_loss} only rescales the loss and has no effect on the solution set.

\paragraph{Estimating $H$ with importance sampling.}
Computing $H$ exactly requires evaluating high-dimensional integrals. 
In EigenVI~\citep{cai2024eigenvi}, the integral is estimated using importance sampling.
Let $\pi(x)$ be a proposal distribution whose support covers that of the target score, and let $\{x_i\}_{i=1}^m$ be i.i.d.\ samples drawn from $\pi$. Then $H$ can be approximated by
\begin{equation}
    \bar{H}
    :=
    \frac{1}{m}\sum_{i=1}^m
    \frac{H_1(x_i)-2H_2(x_i)+4H_3(x_i)}{\pi(x_i)}.
    \label{eq:H_estimation}
\end{equation}
In EigenVI, $\pi$ is chosen as the uniform distribution over a compact domain in $\mathbb{R}^D$. However, when $\pi$ is poorly matched to $p$, the variance of $\bar{H}$ would scale exponentially with $D$, reflecting the \emph{curse of dimensionality}.
This motivates exploiting the locality of $H$ by estimating its local terms separately, which forms the basis of our finite-query analysis.

\section{Main Results}
In this section, we characterize the approximation complexity of MPO-BMs
from both negative and positive perspectives.
We first establish a worst-case limitation: constructing a polynomial-bond
MPO-BM approximation to an arbitrary continuous target with bounded KL error
is NP-hard.
We then identify a positive regime in which MPO-BMs admit polynomial-bond
approximations with provable KL guarantees.
Building on the local importance sampling, we further show that polynomially many local score queries suffice to preserve these guarantees.


\subsection{Hardness of Approximation in the Worst Case}
Prior studies mainly focus on the expressiveness of Born machines~\citep{glasser2019expressive,loconte2025sum}.
Here, we instead study their approximation complexity: whether a target density can be approximated up to a given error by a tractable model.
To begin with, we define the $\epsilon$-approximator w.r.t. the KL divergence as the metric:
\begin{definition}[\textbf{$\epsilon$-KL-Approximation}]\label{def:KL_approximator}
    Let $p, q$ be probability density functions.
    We say that $q$ is an $\epsilon$-KL-approximator of $p$ if $D_{\mathrm{KL}}(q \Vert p) < \epsilon$
    for some $\epsilon > 0$.
\end{definition}

For the hardness result, the target is specified by an efficiently
evaluable unnormalized representation $\tilde p\geq0$, with
$0<Z:=\int\tilde p(x)\,dx<\infty$.
Approximation error is measured against the normalized density
$p=\tilde p/Z$, while $Z$ is not provided as part of the input.

\begin{theorem}[\textbf{Hardness of KL approximation}]
\label{thm:kl-hardness}
For any fixed $0<\epsilon<1/8$, it is NP-hard to construct,
from $\tilde p$, an $\epsilon$-KL-approximator of $p$ that
admits tractable marginalization.
\end{theorem}
The proof is inspired by~\cite{lelandhardness}, which establishes hardness for probabilistic circuits.
Our result extends this hardness result from discrete mass functions to \emph{continuous} densities.
At a high level, we reduce SAT to the approximation problem by constructing, for each Boolean formula, a continuous target density whose probability mass over a specific event encodes satisfiability.
An accurate approximation with tractable marginalization would allow this probability to be computed efficiently, and hence would decide SAT in polynomial time.

The key step compared to the prior work is to embed Boolean assignments into disjoint regions of $\mathbb{R}^D$ using non-overlapping bump functions.
The Boolean formula is then encoded through polynomial operations, so that the resulting density concentrates mass on regions corresponding to satisfying assignments.
This preserves the SAT structure while producing a valid continuous density.
The full proof is deferred to Appendix~\ref{sec:proof-kl-hardness}.

Since Theorem~\ref{thm:kl-hardness} applies to any probabilistic model with tractable marginalization, it immediately yields a limitation for MPO-BMs.
As we have known in Sec.~\ref{sec:MPO_Born_Machine}, MPO-BMs with $R=\mathrm{poly}(D)$ admit tractable marginalization.
We therefore obtain the following corollary, which establishes the
worst-case hardness of constructing polynomial-bond MPO-BM approximations.

\begin{corollary}[\textbf{No Universal Efficient Approximation}]
\label{cor:no_universal_approximation}
Fix $0<\epsilon<1/8$.
Unless $\mathrm{P}=\mathrm{NP}$, no deterministic polynomial-time
algorithm can, for every target specified by an efficiently
evaluable unnormalized representation $\tilde p$, output an
MPO-BM $q_\rho$ with bond dimension $R=\mathrm{poly}(D)$ satisfying
$
D_{\mathrm{KL}}(q_\rho\Vert p)<\epsilon,\, p=\tilde p/Z.
$
\end{corollary}

\begin{remark}[\textbf{On Fisher-divergence hardness}]
The above reduction establishes hardness under KL divergence, but does not
directly extend to Fisher divergence.
While Lemma~\ref{thm:upper_bound_of_Fisher_div_LSI} shows that Fisher
divergence controls KL divergence under an LSI, the SAT-embedding
construction produces densities with separated high-probability regions.
Such bottleneck structures can lead to degenerate LSI constants, preventing
a direct transfer of the KL-hardness result to Fisher divergence.
Establishing Fisher-divergence hardness therefore requires either additional
assumptions or a different reduction with better geometric regularity.
\end{remark}

\subsection{Guaranteed Efficient Approximation under Structural Targets}
We next identify a positive regime in which MPO-BMs admit
compact approximations with provable KL bounds.
In this regime, both the required bond dimension and the number of queries scale polynomially in $D$, yielding provable guarantees on KL approximation error.

\subsubsection{From SBVI to a mixed-ground-state problem}
To establish the provable guarantee for MPO-BM, we consider the following generalization of EigenVI introduced in~\eqref{eq:eigenvi_eigenform}.
By removing the rank restriction to $\rho$, we study the following optimization problem:
\begin{equation}\label{eq:EigenVI_extended}
    \min_{\rho\succeq 0} \operatorname{tr}(\rho H),
    \quad \text{s.t. } \operatorname{tr}(\rho)=1,
\end{equation}
\textcolor{revision}{where $H$ is constructed from~\eqref{eq:def_H} and determined by the target $p$ and the feature map $\Phi$.}
The following lemma shows that \eqref{eq:EigenVI_extended} provides a \emph{linear} surrogate for Fisher-divergence minimization.

\begin{lemma}[\textbf{Fisher-divergence sandwich bound}]
\label{lem:mixed_fisher}
Let $\rho^*$ be any minimizer of~\eqref{eq:EigenVI_extended}.
Then, for every $\theta$ with $\|\theta\|_2=1$,
\begin{equation}
\label{eq:mixed_pure_fisher}
\begin{aligned}
D_F(q_{\rho^*}\Vert p)
\leq \operatorname{tr}(\rho^*H)\leq D_F(q_\theta\Vert p).
\end{aligned}
\end{equation}
\end{lemma}

As shown in Lemma~\ref{lem:mixed_fisher}, the linear loss in~\eqref{eq:EigenVI_extended} upper-bounds the
Fisher divergence of the resulting density at the optimum.
Moreover, its optimal value is no larger than the Fisher divergence achieved
by any rank-one model given in EigenVI in the same feature space.
Note that, under $LSI(c)$, this upper bound can further control KL divergence
through Lemma~\ref{thm:upper_bound_of_Fisher_div_LSI}.

Let $\lambda_1(H)\leq\cdots\leq\lambda_{K^D}(H)$ denote the eigenvalues
of $H$, with $\lambda_{\min}:=\lambda_1(H)$ of multiplicity $r$.
Let $P$ denote the projector onto the eigenspace associated with
$\lambda_{\min}$.
It follows immediately that the minimizers of~\eqref{eq:EigenVI_extended}
are precisely the density operators supported on this eigenspace.
In particular, the maximally mixed solution is given as follows:
\begin{equation}
    \rho^*:=\frac{P}{r}.
\end{equation}
Such a solution reveals a different perspective on the SBVI problem:
in the language of physics, the minimum eigenspace of the Hamiltonian $H$
is precisely its \emph{ground space}.
The SBVI approximation problem can therefore be recast as that of
approximating the maximally mixed ground state $\rho^*$ by an MPO and
translating the resulting energy control into a KL guarantee for
the induced MPO-BM.

\paragraph{Spectral Width.}
For the approximation bound in the next subsection, we use the shift-invariant spectral width
\begin{equation}
    \Lambda
    :=\|H-\lambda_{\min}I\|_{S^\infty}
    =\lambda_{\max}(H)-\lambda_{\min}(H),
\end{equation}
which measures the spectral range of $H$ above its ground-state energy.
An important property is that, for a local Hamiltonian $H=\sum_{e\in\mathcal E}h_e$ with
$|\mathcal E|\leq aD$ and $\|h_e\|_{S^\infty}\leq J$, where $a$ and $J$
are independent of $D$, we have
\begin{equation}
    \Lambda
    \leq \|H\|_{S^\infty}
    \leq \sum_{e\in\mathcal E}\|h_e\|_{S^\infty}
    \leq aJD
    =\mathcal{O}(D).
\end{equation}
Note that we focus on the nontrivial case $\Lambda>0$ in the paper; when $\Lambda=0$, $H$ gets trivially proportional to the identity and every density operator is a ground state.

\subsubsection{MPO-BM VI with infinite queries}
We first consider the population setting, where $H$ is constructed exactly from the target distribution, i.e., we assume access to the exact score function and infinitely many queries.

Under this setting, the ground-state perspective allows us to leverage tools from many-body physics.
In particular, when $H$ is local and has a constant spectral gap, area-law results imply that its ground state admits an efficient MPO approximation.
The following theorem formalizes this intuition.

\begin{theorem}[\textbf{Polynomial-bond approximation}]
\label{thm:model_complexity}
Let $p$ be a strictly positive $C^1$ density satisfying
$LSI(c)$ with $c^{-1}\leq\operatorname{poly}(D)$, and let
$\Phi$ be a $C^1$ orthonormal tensor-product feature map.
Suppose that $H$ is a one-dimensional local Hamiltonian
with fixed local dimension $K$, uniformly bounded
interaction range and strength, and a dimension-independent
spectral gap above its ground space.
Then, for every $\epsilon>0$, there exists a density operator
$\rho_{\mathrm{mpo}}\in\mathrm{MPO}(R)$ with
$
R\leq\operatorname{poly}(D,\epsilon^{-1})
$
such that the induced density
$q_{\mathrm{mpo}}(x)=
\Phi(x)^\top\rho_{\mathrm{mpo}}\Phi(x)$ satisfies
\begin{equation}
D_{\mathrm{KL}}(q_{\mathrm{mpo}}\Vert p)
\leq\epsilon+\frac{\lambda_{\min}}{2c}.
\label{eq:upper_bound_polulation}
\end{equation}
\end{theorem}

In the proof, the key ingredient is the area law of~\cite{arad2026area}, which guarantees a polynomial-bond-dimension MPO-state approximation of the maximally mixed ground state of a one-dimensional gapped local Hamiltonian.
Its energy controls the ensemble-averaged pure-state Fisher objective, which is translated to a KL guarantee by convexity and Lemma~\ref{thm:upper_bound_of_Fisher_div_LSI}.
The full proof is given in Appendix~\ref{sec:proof_model_complexity}.

\begin{remark}[\textbf{Irreducible term in the bound}]
We can see that, for a fixed target $p$ and feature map $\Phi$, the bound
in~\eqref{eq:upper_bound_polulation} contains a
term $\lambda_{\min}/(2c)$ that does not vanish as the MPO
approximation tolerance $\epsilon$ tends to zero.
This term reflects the possibly nonzero minimum of the loss~\eqref{eq:EigenVI_extended}.
It vanishes whenever $\sqrt{p}$ lies in the span of the
components of $\Phi$, in which case the theorem guarantees KL error at most $\epsilon$.
\end{remark}

Theorem~\ref{thm:model_complexity} gives the approximation
guarantee when locality is combined with the stated LSI, constant gap, and
bounded-local-term assumptions.
A naturally following-up question is thus for what $p$ the induced Hamiltonian $H$ is local?
We next characterize a broad class of target densities for which this condition holds.
\begin{proposition}\label{thm:MRF}
If $p$ is a path-graph Markov random field (MRF), then the induced matrix $H$ is a one-dimensional $k$-local Hamiltonian with $k\leq 3$.
\end{proposition}
The proof is given in Appendix~\ref{sec:proof_theorem_MRF}.
This result identifies a concrete class of targets for which the
induced Hamiltonian is local.
Path-graph MRFs form a structured
family that includes first-order Markov chains and chain-structured
pairwise energy models, such as one-dimensional Ising-type models.

\paragraph{Sufficient conditions for a spectral gap.}
The spectral gap of the induced Hamiltonian depends on
both the target density and the feature space.
The following proposition relates this gap to the
log-Sobolev constant and the minimum population energy.

\begin{proposition}[\textbf{A sufficient spectral-gap condition}]
\label{prop:spectral_gap_sufficient}
Let $p$ be a strictly positive $C^1$ density satisfying
$LSI(c)$ for some $c>0$, and let $\Phi$ be a real
$C^1$ orthonormal feature map with $N=K^D\geq2$
components.
Assume that the integrals defining $H$
in~\eqref{eq:def_H} are finite.
Then
\begin{equation}
\lambda_2(H)\geq4c.
\label{eq:second_eigenvalue_lsi}
\end{equation}
Consequently, if $\lambda_{\min}(H)<4c$, the ground
state is unique and
\begin{equation}
\gamma_D
:=\lambda_2(H)-\lambda_{\min}(H)
\geq4c-\lambda_{\min}(H)>0.
\label{eq:gap_from_lsi}
\end{equation}
In particular, if $\sqrt p$ lies in the span of the
components of $\Phi$, then $\lambda_{\min}(H)=0$
and $\gamma_D\geq4c$.
\end{proposition}



The proof is given in
Appendix~\ref{sec:proof_spectral_gap_sufficient}.
Proposition~\ref{prop:spectral_gap_sufficient}
connects the distributional and Hamiltonian assumptions through two
quantities: the LSI constant controls the second eigenvalue of $H$, while
$\lambda_{\min}(H)$ reflects the approximation error induced by the chosen
feature space. In particular, a dimension-independent spectral gap is
guaranteed if
\begin{equation}
    4c-\lambda_{\min}(H)\geq\gamma_0>0
    \qquad\text{uniformly in }D,
    \label{eq:gap_condition}
\end{equation}
for some constant $\gamma_0$ independent of $D$.
It is worth noting that, this condition is sufficient but \emph{not} necessary.
The LSI only provides the lower bound $\lambda_2(H)\geq4c$, while
the actual second eigenvalue may be much larger.
Therefore, failure of~\eqref{eq:gap_condition} means only that this
LSI-based bound cannot certify a dimension-independent spectral gap;
it does not imply that the actual gap
$\lambda_2(H)-\lambda_{\min}(H)$ closes.

\subsubsection{MPO-BM VI with finite queries}
We next consider the setting in which $H$ must be estimated
from finitely many score queries.

\paragraph{Estimating $H$ from local score queries.}
Following the locality assumption,
let each score component $s_j(x)$ of $p$ depend only
on a known interval $C_j=\{\ell_j,\ldots,u_j\}$ of $\mathcal{O}(1)$
coordinates containing $j$, which can be evaluated using
$x_{C_j}$.
For path-graph MRFs for example, we may take
$C_j=\{j-1,j,j+1\}\cap[D]$, as in
Proposition~\ref{thm:MRF}.

Define the local feature vector
$\Phi_{C_j}(x_{C_j})
:=\bigotimes_{d\in C_j}\phi(x_d)$
and the local matrix-valued integrand according to \eqref{eq:def_H} and \eqref{eq:H}:
\begin{equation}
\label{eq:local_gram_integrand}
\begin{aligned}
b_j&:=2\partial_{x_j}\Phi_{C_j}-s_j\Phi_{C_j},\quad{}h_j(x_{C_j})&:=b_jb_j^\top,
\end{aligned}
\end{equation}
where all quantities on the right-hand side are evaluated
at $x_{C_j}$.
Here $b_j\in\mathbb R^{K^{|C_j|}}$ and
$h_j(x_{C_j})\in
\mathbb R^{K^{|C_j|}\times K^{|C_j|}}$.

By orthonormality of the basis functions, integrating
out the coordinates outside $C_j$
in~\eqref{eq:def_H} yields
\begin{equation}
\begin{aligned}
A_j
&:=\int_{\mathbb R^{|C_j|}}
h_j(x_{C_j})\,dx_{C_j},\\
H
&=\sum_{j=1}^{D}
I_K^{\otimes(\ell_j-1)}
\otimes A_j
\otimes I_K^{\otimes(D-u_j)}.
\end{aligned}
\end{equation}
The identity factors act on the remaining coordinates,
with $I_K^{\otimes0}=1$.

For each $j$, choose a proposal density $\pi_j$ on
$\mathbb R^{|C_j|}$ (e.g., uniform) that is positive wherever $h_j$ is nonzero.
Draw $X_{j,1},\ldots,X_{j,m}$ independently from $\pi_j$,
with independent sampling across terms, and define
\begin{equation}
\label{eq:local_H_estimator}
\begin{aligned}
\bar A_j
&:=\frac{1}{m}\sum_{i=1}^{m}
\frac{h_j(X_{j,i})}{\pi_j(X_{j,i})},\\
\bar H
&:=\sum_{j=1}^{D}
I_K^{\otimes(\ell_j-1)}
\otimes\bar A_j
\otimes I_K^{\otimes(D-u_j)}.
\end{aligned}
\end{equation}
Then $\mathbb E[\bar A_j]=A_j$ and
$\mathbb E[\bar H]=H$.
Since the basis functions and their derivatives are known,
each evaluation of $h_j$ requires only one scalar
score query $s_j(X_{j,i})$, giving a total budget $B=Dm$.

Using this estimator, we consider the empirical
optimization problem:
\begin{equation}\label{eq:EigenVI_extended_estimated}
\begin{aligned}
&\min_{\rho}\, \operatorname{tr}(\rho \bar{H}),\\
&\text{s.t.}\;\rho = \frac{1}{r} U U^\top,\quad
U^\top U = I_r,\\
&\phantom{\text{s.t.}\;}U \in \mathbb{R}^{K^D \times r},
\end{aligned}
\end{equation}
where $r$ is the multiplicity of the population minimum eigenvalue and $I_r$
is the $r\times r$ identity. 

Note that we assume $r$ is known and keep it fixed in the
optimization, since estimation noise can split a degenerate
minimum eigenvalue and cause unrestricted minimization to
select a single eigenvector.
We therefore preserve the entire subspace spanned by the $r$
lowest empirical eigenvectors, which remains close to the
population ground space when the estimation error is
sufficiently small relative to the population spectral gap.

The following lemma gives a sufficient query budget for
uniform local importance sampling when the integrands have bounded
support and uniformly bounded spectral norms.

\begin{lemma}[\textbf{Concentration of the local estimator}]
\label{thm:concentration}
Consider the estimator $\bar H$
in~\eqref{eq:local_H_estimator}.
Suppose that $|C_j|\leq C$ and each local integrand
$h_j$ is supported on
$
\mathcal X_j=[-\ell/2,\ell/2]^{|C_j|},
$
with $\|h_j(x_{C_j})\|_{S^\infty}\leq J$ almost everywhere,
where $C,\ell,J$ are fixed constants.
For each $j$, use the uniform proposal on $\mathcal X_j$
with $m$ independent samples, independently across terms.
Then the local estimator takes the form
$
\bar A_j
=\frac{\ell^{|C_j|}}{m}
\sum_{i=1}^{m}h_j(X_{j,i}).
$
There exists a constant $C_0>0$, depending only on
$C,\ell,J$, such that, for every
$\epsilon,\delta\in(0,1)$,
$
\Pr\!\left\{
\|\bar H-H\|_{S^\infty}\leq\epsilon
\right\}\geq1-\delta
$
whenever
\begin{equation}
m\geq
C_0\frac{D^2}{\epsilon^2}
\log\!\left(\frac{2D\operatorname{poly}(K)}{\delta}\right).
\label{eq:lower_bound_concentration}
\end{equation}
Here $\operatorname{poly}(K)$ bounds the local matrix
dimensions, with degree determined by $C$.
Consequently, a sufficient total local score-query budget
$B=Dm$ can be chosen as
\begin{equation}
B=O\!\left(
\frac{D^3}{\epsilon^2}
\log\!\frac{2D\operatorname{poly}(K)}{\delta}
\right).
\end{equation}
\end{lemma}

The proof is given in Appendix~\ref{sec:proof_concentraction} based on Tropp's Matrix Bernstein inequality~\citep{tropp2015introduction}.

Lemma~\ref{thm:concentration} gives a polynomial sufficient local-query
budget under its bounded compact-support assumptions.  By contrast, the
analogous worst-case matrix-Bernstein guarantee for uniform global sampling used in EigenVI
contains the volume factor $\ell^{2D}$; see
Lemma~\ref{thm:concentration_of_global_H}.

Combining the concentration result with the MPO approximation guarantee yields the following theorem.

\begin{theorem}[\textbf{Finite-query approximation and stability}]
\label{thm:query_complexity}
Assume the conditions of
Theorem~\ref{thm:model_complexity} and
Lemma~\ref{thm:concentration}.
Let $\gamma_D$ denote the spectral gap above the ground
space of $H$, and let $\bar\rho$ solve
\eqref{eq:EigenVI_extended_estimated}
using the known ground-space multiplicity $r$.
For $\epsilon>0$ and $\delta\in(0,1)$, define
$
\eta
:=\min\left\{
\frac12,\,
\frac{\gamma_D}{4},\,
\frac{\gamma_D c\epsilon}{c\epsilon+2\Lambda}
\right\}.
$
Suppose that $\bar H$ is constructed using $m$ samples
per local term and that the total query budget
$B=Dm$ satisfies
\begin{equation}
B\geq
C_1\frac{D^3}{\eta^2}
\log\!\left(
\frac{2D\,\operatorname{poly}(K)}{\delta}
\right),
\end{equation}
where $C_1>0$ is a sufficiently large constant depending
only on the fixed local parameters in
Lemma~\ref{thm:concentration}.
Then, with probability at least $1-\delta$,
\begin{equation}
\begin{aligned}
\|\bar\rho-\rho^*\|_{S^1}
&\leq\frac{c\epsilon}{\Lambda},\,
D_{\mathrm{KL}}(q_{\bar\rho}\Vert p)
&\leq\frac{\lambda_{\min}}{2c}+\frac{\epsilon}{2}.
\end{aligned}
\end{equation}

Moreover, there exists a density operator
$\rho_{\mathrm{mpo}}\in\mathrm{MPO}(R)$,
independent of the query samples, with
$
R\leq\operatorname{poly}(D,\epsilon^{-1}),
$
such that, on the same event,
\begin{equation}
\begin{aligned}
\|\rho_{\mathrm{mpo}}-\bar\rho\|_{S^1}
&\leq\frac{2c\epsilon}{\Lambda},\,D_{\mathrm{KL}}(q_{\mathrm{mpo}}\Vert p)
&\leq\frac{\lambda_{\min}}{2c}+\frac{\epsilon}{2}.
\end{aligned}
\label{eq:mpo_rho_bar_bound}
\end{equation}
\end{theorem}

\begingroup\color{revision}
The proof is given in Appendix~\ref{sec:proof_query_complexity}.
Compared with Theorem~\ref{thm:model_complexity},
Theorem~\ref{thm:query_complexity} implies that polynomially many local-integrand
evaluations should suffice to make the empirical low-energy state close to a
polynomial-bond-dimension population approximator.  
Note that this result guarantees the existence of an accurate MPO
approximation from sufficiently many local queries, but does not provide
an efficient algorithm for constructing such an MPO from $\bar H$.


\section{Related Work}

\paragraph{Tensor-network-based probabilistic modeling.}
Tensor networks support tractable probabilistic
modeling~\citep{han2018unsupervised,cheng2019tree,miller2021tensor}.
Existing theory has primarily characterized their representational
properties. Prior work compares the bond dimensions required by
nonnegative, Born, and locally purified
factorizations~\citep{glasser2019expressive}, connects tensor
representations to probabilistic
circuits~\citep{loconte2025relationship,loconte2025sum}, and uses
orthogonality to reduce marginalization costs in squared
models~\citep{loconte2026square}.
For continuous densities, TTDE learns tensor-train density models from
samples~\citep{novikov2021tensor}, hierarchical sketching provides
finite-sample Frobenius-norm guarantees under rank and sketching
assumptions~\citep{peng2023generative}, and MPS Born machines admit
Jensen--Shannon approximation bounds for smooth densities on compact
domains~\citep{meiburg2025generative}.
Relatedly, for discrete MRFs, compact TT representations of the energy
need not yield compact representations of the corresponding probability
tensor~\citep{novikov2014putting}.
Unlike prior work, we study the complexity of constructing accurate
MPO-BM approximations.
We prove worst-case NP-hardness even allowing constant KL approximation
error, and identify a more clear regime with polynomial
bond dimension and local score-query complexity under KL guarantees.

\paragraph{Score-based variational inference.}
Score matching uses log-density gradients to avoid unknown
normalizing constants~\citep{hyvarinen2005estimation} and is central
to score-based diffusion models~\citep{song2020score}.
In variational inference, target scores have been used to fit
approximations through Gaussian score matching and batch-and-match
updates~\citep{modi2023variational,caibatch}.
EigenVI reduces Fisher-divergence minimization over an orthogonal
function expansion to an eigenvalue
problem~\citep{cai2024eigenvi}.
More recently, \citet{cong2026quanvi} extend this rank-one formulation
to a linear Fisher surrogate over mixed states and develop a
quantum-inspired MPO-based algorithm.
These works introduce the algorithmic framework, but leave open when
the resulting variational problem admits an accurate and compact MPO
solution. We address this question theoretically. By interpreting the
linear Fisher surrogate as a ground-state problem, we relate properties
of the target distribution to locality and spectral properties of the
induced Hamiltonian. This connection yields polynomial-bond MPO
approximations with reverse-KL guarantees under explicit conditions,
together with polynomial local score-query bounds for their statistical
stability.


\section{Discussion and Conclusion}

\paragraph{Feature choice is part of the approximation guarantee.}
Proposition~\ref{prop:spectral_gap_sufficient} highlights that the
spectral-gap condition depends not only on the regularity of the target,
but also on how well the chosen feature space represents it.
While the LSI constant provides the lower bound
$\lambda_2(H)\geq4c$, obtaining a dimension-independent gap from this
bound additionally requires $\lambda_{\min}(H)$ to remain uniformly
below $4c$.
Thus, the LSI condition alone is insufficient; the chosen feature space
must also keep $\lambda_{\min}(H)$ sufficiently small.
A natural future direction is therefore to develop target-adapted local features
that keep $\lambda_{\min}(H)$ controlled as $D$ grows, while tracking the
resulting representation and query complexity as a function of $K$.

\paragraph{Degeneracy leaves room for sharper guarantees.}
Our analysis covers both nondegenerate and degenerate ground spaces, but
the latter leaves room for sharper KL guarantees.
Although the polynomial-bond approximation result continues to hold when
$r\geq2$, Proposition~\ref{prop:spectral_gap_sufficient} implies
$\lambda_{\min}=\lambda_2(H)\geq4c$ in this case.
As a result, the term $\lambda_{\min}/(2c)$ in our KL guarantee remains
bounded below by a positive constant, so the resulting KL bound need not
be tight for degenerate ground spaces.
This looseness arises in part because, as a linear surrogate,  the loss function $\operatorname{tr}(\rho H)$
averages the Fisher losses of the pure-state components, whereas the
Fisher divergence of their mixture can be smaller.
An important improvement for future work is therefore to characterize this gap more directly and develop guarantees that exploit the mixed-state structure beyond the linear surrogate.

\paragraph{From approximation guarantees to efficient learning.}
Our results establish polynomial-bond representations and polynomial
score-query complexity, but do not provide an efficient algorithm for
constructing the corresponding MPO-BM.
Recent work on QuanVI~\citep{cong2026quanvi} shows that the same mixed-state formulation can be
optimized directly with compact tensor-network parameterizations, providing
empirical evidence of practical scalability.
A natural direction for future work is to bridge this practical success
with our theory by developing provably efficient algorithms that construct
PSD, unit-trace MPOs from local score queries while controlling optimization
error, bond dimension, and computational cost.


\vspace{-0.5cm}
\paragraph{Conclusion.}
We characterized the approximation complexity of continuous MPO-BMs from
both worst-case and structured perspectives.
While constructing KL-accurate approximations is NP-hard in general,
locality, spectral structure, and distributional regularity yield
polynomial-bond approximations with polynomial local score-query complexity.
Together, these results clarify when the expressivity of MPO-BMs can be
translated into provable approximation guarantees and identify the target
and feature structures that make such guarantees possible.

\subsection*{AI Use Statement}
Generative AI tools, including OpenAI Codex, were used to assist with
mathematical reasoning, experiments, coding, literature review, and
manuscript preparation.
The research was developed through repeated human--AI iterations; no
substantive component was incorporated without author review, verification,
and revision.
The authors made all final decisions regarding the mathematical arguments,
experiments, interpretation, and presentation, and take full responsibility
for the theoretical claims, proofs, code, and reported numerical results.

\FloatBarrier
\bibliographystyle{plainnat}
\bibliography{ref}

\clearpage
\appendix
\onecolumn
\thispagestyle{empty}
\aistatstitle{On the Approximation Complexity of\\
Matrix Product Operator Born Machines:\\Supplementary Materials}

\section{Analytic Preliminaries and Proofs}
\label{sec:notation_concepts}

This appendix establishes the connection between the linear variational
objective and reverse-KL approximation. We first specify the convention
for the log-Sobolev inequality, then prove the Fisher-divergence and
spectral bounds used in the main text. Each main-text result is restated
before its proof.

\paragraph{Analytic assumptions.}
For the Fisher and spectral arguments, we take $p$ to be a strictly
positive $C^1$ density, $\Phi$ to be a real $C^1$ orthonormal feature map,
and the integrals defining $H$ to be finite. These conditions justify
the score identities and the form-domain arguments in
Section~\ref{sec:proof_mixed_fisher}. The pure-state identity
$D_F(q_\theta\Vert p)=\theta^\top H\theta$ concerns unit vectors
$\theta$, so that $q_\theta$ is normalized. These analytic assumptions
are separate from the input model of the hardness reduction in
Section~\ref{sec:proof-kl-hardness}.

\subsection{Density operators, norms, and the LSI form domain}
Throughout the analytic arguments, $N=K^D$, all feature maps and matrices
are real, and integrals without an indicated domain are over
$\mathbb R^D$ with respect to Lebesgue measure. Eigenvalues are ordered
increasingly and counted with multiplicity. An ensemble is a finite collection
$\{(\alpha_i,\theta_i)\}_{i=1}^{L}$ with $\alpha_i\geq0$,
$\sum_i\alpha_i=1$, and $\|\theta_i\|_2=1$. It defines the density
operator $\rho=\sum_i\alpha_i\theta_i\theta_i^\top$. Conversely, the
spectral decomposition of every real PSD, unit-trace matrix supplies such
an ensemble with at most $N$ components. The ensemble need not be unique,
and a general ensemble need not have exactly $N$ components.

For a matrix $A$ with singular values $\sigma_i(A)$, the Schatten norms are
\[
 \|A\|_{S^p}=\Bigl(\sum_i\sigma_i(A)^p\Bigr)^{1/p}
 \quad(1\leq p<\infty),\qquad
 \|A\|_{S^\infty}=\max_i\sigma_i(A).
\]
Thus $S^1$, $S^2$, and $S^\infty$ give the trace, Frobenius, and spectral
norms. We use the duality bound
$|\operatorname{tr}(AB)|\leq\|A\|_{S^1}\|B\|_{S^\infty}$ for symmetric
matrices of the same size.

\begin{definition}[Log--Sobolev inequality and its form domain]
\label{def:LSI}
For a probability density $p$ and an integrable nonnegative function $g$,
let
\[
 \operatorname{Ent}_p(g)
 =\int g\log g\,p\,dx
  -\left(\int gp\,dx\right)\log\left(\int gp\,dx\right),
 \qquad 0\log0:=0.
\]
We write $H^1(p)$ for the weighted Sobolev form domain, obtained by
completing smooth functions with finite norm
\[
 \|f\|_{H^1(p)}^2=\int\bigl(f^2+\|\nabla f\|_2^2\bigr)p\,dx.
\]
For strictly positive $C^1$ densities, this is equivalently the space of
locally weakly differentiable functions for which the displayed integral
is finite. The weak gradient allows functions that are not classically
differentiable everywhere. We say that $p$ satisfies $LSI(c)$ if
\[
 \operatorname{Ent}_p(f^2)
 \leq\frac{2}{c}\int\|\nabla f\|_2^2p\,dx
 \qquad\text{for all }f\in H^1(p).
\]
It suffices to establish this inequality for smooth functions and extend
it to the form domain by approximation.
\end{definition}

This normalization of $LSI(c)$ is used throughout the paper. The notation
$q\ll p$ means that the probability measure with density $q$ is absolutely
continuous with respect to that with density $p$. For such a density $q$
with $\sqrt{q/p}\in H^1(p)$, the Fisher divergence can be written as
\begin{equation}
 D_F(q\Vert p)
 =4\int\left\|\nabla\sqrt{q/p}\right\|_2^2p\,dx.
 \label{eq:app_weak_fisher}
\end{equation}
When classical scores are defined, this is the usual
$\int q\|\nabla\log q-\nabla\log p\|_2^2dx$.
Equation~\eqref{eq:app_weak_fisher} specifies the interpretation at zeros
of a Born density; positivity of every pure-state density is not required.

\subsection{Fisher divergence controls reverse KL}
\label{sec:proof_fisher_kl}
\begin{restated}[Lemma~\ref{thm:upper_bound_of_Fisher_div_LSI}]
Let $q\ll p$ and assume $\sqrt{q/p}\in H^1(p)$. If $p$ satisfies
$LSI(c)$ for some $c>0$, then
\[
 D_{\mathrm{KL}}(q\Vert p)\leq\frac{1}{2c}D_F(q\Vert p),
 \qquad
 D_{\mathrm{KL}}(q\Vert p)=\int q\log(q/p)\,dx.
\]
\end{restated}
\begin{proof}
Set $f=\sqrt{q/p}$. Since $\int f^2p\,dx=1$, we have
$\operatorname{Ent}_p(f^2)=D_{\mathrm{KL}}(q\Vert p)$.
Applying $LSI(c)$ to $f$ and using~\eqref{eq:app_weak_fisher} gives
\[
 D_{\mathrm{KL}}(q\Vert p)
 \leq\frac{2}{c}\int\|\nabla f\|_2^2p\,dx
 =\frac{1}{2c}D_F(q\Vert p).
\]
The assumed form-domain membership justifies this application even when
$f$ is not classically differentiable everywhere.
\end{proof}

\subsection{Regularity and the mixed-density Fisher objective}
\label{sec:proof_mixed_fisher}
The next lemma makes explicit why the smoothness and integrability
assumptions in the approximation theorems suffice for the Fisher-to-KL
step. It also distinguishes the linear energy from the Fisher divergence
of a mixed density.

\begin{lemma}[Energy certificate for a density operator]
\label{lem:energy_kl_certificate}
Let $p$ be a strictly positive $C^1$ density and let
$\Phi:\mathbb R^D\to\mathbb R^N$ be a real $C^1$ feature map satisfying
$\int\Phi\Phi^\top dx=I_N$. Suppose the integrals defining $H$ are finite.
For any density operator and ensemble decomposition
$\rho=\sum_i\alpha_i\theta_i\theta_i^\top$,
the induced densities $q_i=|\theta_i^\top\Phi|^2$ and
$q_\rho=\sum_i\alpha_iq_i$ are absolutely continuous with respect to $p$
and satisfy $\sqrt{q_i/p},\sqrt{q_\rho/p}\in H^1(p)$.
Moreover,
\begin{equation}
 D_F(q_\rho\Vert p)
 \leq\sum_i\alpha_iD_F(q_i\Vert p)
 =\operatorname{tr}(\rho H).
 \label{eq:app_fisher_energy}
\end{equation}
If $p$ additionally satisfies $LSI(c)$, then
\begin{equation}
 D_{\mathrm{KL}}(q_\rho\Vert p)
 \leq\frac{\operatorname{tr}(\rho H)}{2c}.
 \label{eq:app_energy_kl}
\end{equation}
\end{lemma}
\begin{proof}
\emph{Step 1: write the Hamiltonian as a Gram matrix.}
With $s=\nabla\log p$ and the Jacobian convention of~\eqref{eq:H}, define
\[
 B(x)=2\dot\Phi(x)-s(x)\Phi(x)^\top\in\mathbb R^{D\times N}.
\]
Expanding $B^\top B$ gives exactly the integrand in~\eqref{eq:def_H}, hence
\begin{equation}
 H=\int B(x)^\top B(x)\,dx\succeq0.
 \label{eq:app_gram_H}
\end{equation}

\emph{Step 2: verify the form domain for pure states.}
For a unit vector $\theta$, put
$\psi_\theta=\theta^\top\Phi$ and $f_\theta=\psi_\theta/\sqrt p$.
Positivity and smoothness of $p$ imply local differentiability, and
orthonormality gives
\[
 \int f_\theta^2p\,dx=1,\qquad
 4\int\|\nabla f_\theta\|_2^2p\,dx
 =\int\|2\nabla\psi_\theta-s\psi_\theta\|_2^2dx
 =\theta^\top H\theta<\infty.
\]
Thus $f_\theta\in H^1(p)$. The Sobolev chain rule gives
$|f_\theta|=\sqrt{q_\theta/p}\in H^1(p)$ and
$\|\nabla|f_\theta|\|_2=\|\nabla f_\theta\|_2$ almost everywhere.
At zeros this uses the fact that a Sobolev function has zero weak gradient
almost everywhere on its zero set. Consequently
$D_F(q_\theta\Vert p)=\theta^\top H\theta$, including when the amplitude
has zeros. Absolute continuity follows from $p>0$.

\emph{Step 3: pass to the mixed density.}
Let $F(x)=(\sqrt{\alpha_i}f_{\theta_i}(x))_{i=1}^L$. Then
$\sqrt{q_\rho/p}=\|F\|_2$. The chain rule for the Euclidean norm gives
\[
 \left\|\nabla\sqrt{q_\rho/p}\right\|_2^2
 \leq\sum_i\alpha_i\|\nabla f_{\theta_i}\|_2^2
 \quad\text{almost everywhere}.
\]
The right-hand side is integrable, so the mixed ratio also lies in
$H^1(p)$. Integrating proves the inequality in~\eqref{eq:app_fisher_energy}.
The equality follows from
$\sum_i\alpha_i\theta_i^\top H\theta_i=\operatorname{tr}(\rho H)$.
Finally, Lemma~\ref{thm:upper_bound_of_Fisher_div_LSI} proves
\eqref{eq:app_energy_kl}.
\end{proof}

\begin{restated}[Lemma~\ref{lem:mixed_fisher} (Fisher-divergence sandwich bound)]
Under the regularity assumptions of Lemma~\ref{lem:energy_kl_certificate},
let $\rho^*$ be any minimizer of~\eqref{eq:EigenVI_extended}. Then, for
every unit vector $\theta$,
\[
 D_F(q_{\rho^*}\Vert p)
 \leq\operatorname{tr}(\rho^*H)
 \leq D_F(q_\theta\Vert p).
\]
\end{restated}
\begin{proof}
Lemma~\ref{lem:energy_kl_certificate} gives the first inequality.
The rank-one matrix $\theta\theta^\top$ is feasible for
\eqref{eq:EigenVI_extended}, so optimality and the pure-state identity give
\[
 \operatorname{tr}(\rho^*H)
 \leq\operatorname{tr}(\theta\theta^\top H)
 =D_F(q_\theta\Vert p).
\]
In particular, the optimal linear objective equals $\lambda_{\min}(H)$.
Every density operator supported on the corresponding eigenspace attains
this value. Selecting $\rho^*=P/r$ fixes one such optimizer without
asserting that it minimizes the nonlinear Fisher divergence over all
mixed densities.
\end{proof}

\subsection{Spectral width of a local Hamiltonian}
\label{sec:proof_spectral_width}
\begin{restated}[Spectral-width bound]
Suppose $H\succeq0$ admits a decomposition $H=\sum_{e\in\mathcal E}h_e$
with $|\mathcal E|\leq aD$ and $\|h_e\|_{S^\infty}\leq J$, where $a,J$
are independent of $D$. Then
\begin{equation}
 \Lambda:=\lambda_{\max}(H)-\lambda_{\min}(H)\leq aJD=O(D).
 \label{eq:app_spectral_width_bound}
\end{equation}
\end{restated}
Here $h_e$ denotes a fixed $N\times N$ interaction embedded in the full
feature space, as in the main-text decomposition;
the functions $h_j(x_{C_j})$ used in local sampling are integrands whose
integrals form the local blocks $A_j$.
\begin{proof}
Since $H\succeq0$, its minimum eigenvalue is nonnegative. Therefore
\[
 \Lambda\leq\lambda_{\max}(H)=\|H\|_{S^\infty}
 \leq\sum_{e\in\mathcal E}\|h_e\|_{S^\infty}\leq aJD.
\]
Only positivity of $H$, the number of terms, and their norm bounds are
needed here; individual interactions need not be PSD.
\end{proof}

\subsection{A sufficient spectral-gap condition}
\label{sec:proof_spectral_gap_sufficient}
\begin{restated}[Proposition~\ref{prop:spectral_gap_sufficient}]
Let $p$ be a strictly positive $C^1$ density satisfying $LSI(c)$ for $c>0$, and let
$\Phi$ be a real $C^1$ orthonormal feature map with $N=K^D\geq2$ components.
Assume the integrals defining $H$ are finite. Then
$\lambda_2(H)\geq4c$. If $\lambda_{\min}(H)<4c$, the ground state is unique
and
\[
 \gamma_D=\lambda_2(H)-\lambda_{\min}(H)
 \geq4c-\lambda_{\min}(H)>0.
\]
If $\sqrt p$ lies in the feature span, then $\lambda_{\min}(H)=0$ and
$\gamma_D\geq4c$.
\end{restated}
\begin{proof}
\emph{Step 1: obtain a variance bound from LSI.}
Write $\operatorname{Var}_p(f)=\int f^2p\,dx-(\int fp\,dx)^2$.
In our convention, $LSI(c)$ implies the Poincar\'e inequality
\begin{equation}
 \operatorname{Var}_p(f)\leq\frac{1}{c}\int\|\nabla f\|_2^2p\,dx.
 \label{eq:app_poincare}
\end{equation}
To check the constant, take a smooth bounded centered function $g\in H^1(p)$ and
apply LSI to $1+tg$. As $t\to0$,
$\operatorname{Ent}_p((1+tg)^2)=2t^2\int g^2p\,dx+o(t^2)$, whereas the
right-hand side of LSI is $(2t^2/c)\int\|\nabla g\|_2^2p\,dx$.
Divide by $2t^2$, let $t\to0$, and extend by truncation and approximation
to $H^1(p)$ to obtain~\eqref{eq:app_poincare}.

\emph{Step 2: bound all but one eigenvalue.}
For $\theta\in\mathbb R^N$, set $f_\theta=\theta^\top\Phi/\sqrt p$ and
\[
 v:=\int\Phi(x)\sqrt{p(x)}\,dx\in\mathbb R^N.
\]
The form-domain argument in Lemma~\ref{lem:energy_kl_certificate} applies
to every $\theta$ by scaling. The vector $v$ contains the coefficients
of the $L^2$ projection of $\sqrt p$ onto the feature span; in particular,
$\int f_\theta p\,dx=\theta^\top v$ and
$\int f_\theta^2p\,dx=\|\theta\|_2^2$.
Orthonormality and~\eqref{eq:app_poincare} therefore give
\[
 \theta^\top H\theta
 =4\int\|\nabla f_\theta\|_2^2p\,dx
 \geq4c\bigl(\|\theta\|_2^2-|\theta^\top v|^2\bigr).
\]
Thus $H\succeq4c(I_N-vv^\top)$. Bessel's inequality gives
$\|v\|_2^2\leq\int p\,dx=1$. The eigenvalues of $I_N-vv^\top$ are
$1-\|v\|_2^2$ and $N-1$ copies of $1$ (with the same conclusion if
$v=0$). Eigenvalue monotonicity in PSD order therefore yields
$\lambda_2(H)\geq4c$.

\emph{Step 3: conclude the gap statements.}
If $\lambda_{\min}(H)<4c$, no second eigenvalue can equal it, so the
minimum is simple and its gap is at least $4c-\lambda_{\min}(H)$.
If $\sqrt p=\theta_0^\top\Phi$, orthonormality implies
$\|\theta_0\|_2=1$ and
$2\nabla\sqrt p-(\nabla\log p)\sqrt p=0$.
Hence $\theta_0^\top H\theta_0=0$.
Together with $H\succeq0$, this proves the final assertion.
\end{proof}

The criterion is sufficient, not necessary, for a spectral gap. A unit trial vector gives
$\lambda_{\min}(H)\leq\theta^\top H\theta$ and can therefore certify the
criterion without knowing the exact minimum. For a family indexed by $D$,
the certified margin $4c_D-\theta_D^\top H_D\theta_D$ must remain bounded
below by a positive constant independently of $D$.
If the minimum has multiplicity $r\geq2$, then instead
$\lambda_{\min}(H)=\lambda_2(H)\geq4c$ and the irreducible term in the
current KL upper bound is at least $2$. This is a restriction on that
certificate, not a lower bound on the actual reverse-KL error.

\section{Computational Hardness}
\label{sec:proof-kl-hardness}

\paragraph{Computational model.}
The target is specified by a finite description of a nonnegative,
unnormalized function $\tilde p:\mathbb R^D\to\mathbb R_+$, with the promise
$0<Z:=\int\tilde p<\infty$; the approximation guarantee refers to
$p=\tilde p/Z$.
The input does not include $Z$ or an oracle for the normalized density.
We use the standard bit model: rational numbers are encoded in binary,
and running time is measured in the length of the target description and
the requested numerical precision.
The hard instances below have polynomial-size descriptions using fixed
piecewise-polynomial functions, rational arithmetic, and comparisons;
their values at rational query points can be computed exactly in time
polynomial in the description length and the query's bit length.
Thus the construction supplies an explicit evaluation procedure, not an
unimplemented oracle.

An output model has a finite description and admits tractable
marginalization if a uniform deterministic algorithm evaluates
\[
 Q(C)=\int_C q(x)\,dx,
 \qquad
 C=\{x:x_d\in[a_d,b_d]\text{ for every }d\in S\},
\]
to any prescribed rational additive accuracy $\tau\in(0,1)$ in time
polynomial in the model description length, the encoding length of the
coordinate cylinder $C$, and $\log(1/\tau)$.
Here $S\subseteq[D]$ and all finite interval endpoints are rational.
Only one fixed interval on one coordinate, at a fixed positive accuracy,
is needed in the reduction.

\begin{restated}[Theorem~\ref{thm:kl-hardness} (Hardness of KL approximation)]
For any fixed $0<\epsilon<1/8$, it is NP-hard to construct,
from $\tilde p$, an $\epsilon$-KL-approximator of $p$ that
admits tractable marginalization.
\end{restated}

Under the preceding computational model, the conclusion means that no
deterministic polynomial-time algorithm can always produce such a model
unless $\mathrm P=\mathrm{NP}$.
The approximation direction is $D_{\mathrm{KL}}(q\Vert p)<\epsilon$.
It is a statement about computing an approximation, not about the
nonexistence of a compact representation for an individual target.

\begin{proof}
We give a polynomial-time reduction from SAT, following the discrete
encoding idea of~\citet{lelandhardness} and replacing Boolean assignments
by disjoint regions carrying continuous densities.
Let $f$ be a Boolean formula in $n$ variables, and set $D=n+1$.
We write $(x,y)\in\mathbb R^n\times\mathbb R$ for the continuous variables.

\paragraph{Step 1: fixed continuous encodings of Boolean values.}
Define the nonnegative, continuous function
\begin{equation}
 \eta(t)=
 \begin{cases}
  3(1-16t^2),& |t|\leq 1/4,\\
  0,& |t|>1/4,
 \end{cases}
 \qquad
 \eta_0(t)=\eta(t),\qquad \eta_1(t)=\eta(t-1).
\end{equation}
Direct integration gives $\int_{\mathbb R}\eta(t)\,dt=1$.
The supports of $\eta_0$ and $\eta_1$ are the disjoint rational intervals
\[
 I_0=[-1/4,1/4],\qquad I_1=[3/4,5/4].
\]
For $z\in\{0,1\}^n$, let
\[
 b_z(x)=\prod_{i=1}^n\eta_{z_i}(x_i).
\]
Each $b_z$ is a continuous probability density, and the supports of
$b_z$ and $b_{z'}$ are disjoint whenever $z\neq z'$.
Use the everywhere-defined decoder
\[
 \sigma(t)=\min\{1,\max\{0,2t-1/2\}\}.
\]
It satisfies $\sigma(t)=z$ on $I_z$ for $z\in\{0,1\}$, including the endpoints;
in particular, it never divides by a vanishing bump value.

\paragraph{Step 2: a polynomial-size unnormalized target.}
Replace the Boolean operations in $f$ by
\[
 \neg a\longmapsto 1-a,\qquad
 a\land b\longmapsto ab,\qquad
 a\lor b\longmapsto a+b-ab.
\]
This gives an arithmetic circuit $F_f$ of size linear in the formula
length, with $F_f(z)=f(z)$ on Boolean inputs.
Define
\begin{align}
 g_f(x)
 &=F_f\bigl(\sigma(x_1),\ldots,\sigma(x_n)\bigr)
      \prod_{i=1}^n\bigl(\eta_0(x_i)+\eta_1(x_i)\bigr),\\
 u(x)&=\prod_{i=1}^n\eta_1(x_i),\\
 \tilde p_f(x,y)&=g_f(x)\eta_1(y)+u(x)\eta_0(y).
 \label{eq:hardness_unnormalized_target}
\end{align}
All functions and rational constants other than the circuit $F_f$ are
fixed, so a description of $\tilde p_f$ can be constructed in polynomial
time without enumerating assignments.

For clarity, exact evaluation also has polynomial bit complexity.
If any coordinate lies outside $I_0\cup I_1$, the product defining $g_f$
vanishes and we return $g_f(x)=0$ without evaluating $F_f$.
Otherwise every decoded input is Boolean, so every Boolean-circuit value
is $0$ or $1$.
The remaining operations evaluate fixed quadratic polynomials and
multiply $O(D)$ rational factors; their bit lengths grow at most
polynomially in the query length and $D$.
This describes a uniform polynomial-time rational evaluation algorithm
for $\tilde p_f$.

\paragraph{Step 3: normalization is used only in the analysis.}
Disjointness of the supports and the decoder property imply the identity
\begin{equation}
 g_f(x)=\sum_{z\in\{0,1\}^n} f(z)b_z(x).
 \label{eq:hardness_assignment_identity}
\end{equation}
On the support of $b_z$, all other assignment densities vanish and the
decoder equals $z$; outside their union, both sides vanish.
Equation~\eqref{eq:hardness_assignment_identity} is an identity used for
analysis, not an expansion constructed by the reduction.
It also shows that $g_f$ is continuous and nonnegative.
Writing $M_f:=\sum_z f(z)$ for the number of satisfying assignments gives
\[
 \int g_f(x)\,dx=M_f,\qquad
 \int u(x)\,dx=1,\qquad
 \int\tilde p_f(x,y)\,dx\,dy=M_f+1.
\]
Thus $\tilde p_f$ satisfies the input promise and defines the continuous
probability density
\[
 p_f(x,y)=\frac{\tilde p_f(x,y)}{M_f+1}.
\]
Neither $M_f$ nor $M_f+1$ is computed or supplied to the approximation
algorithm.

\paragraph{Step 4: one marginal query decides satisfiability.}
Consider the rational coordinate cylinder
\[
 C=\mathbb R^n\times I_1.
\]
Because $\eta_0$ and $\eta_1$ have disjoint supports,
\[
 P_f(C)=\frac{M_f}{M_f+1}
 =0\quad\text{if $f$ is unsatisfiable},
 \qquad
 P_f(C)\geq\frac12\quad\text{otherwise}.
\]
Suppose a deterministic polynomial-time algorithm constructs a model $q$
with $D_{\mathrm{KL}}(q\Vert p_f)<\epsilon$ and tractable marginalization.
Pinsker's inequality yields
\[
 |Q(C)-P_f(C)|
 \leq D_{\mathrm{TV}}(q,p_f)
 \leq\sqrt{\tfrac12D_{\mathrm{KL}}(q\Vert p_f)}
 <\beta_\epsilon:=\sqrt{\epsilon/2}<\frac14.
\]
The positive margin $m_\epsilon:=1/4-\beta_\epsilon$ depends only on the
fixed accuracy $\epsilon$.
Choose once and for all a rational $\tau_\epsilon$ with
$0<\tau_\epsilon<m_\epsilon/2$ and compute $\widehat Q(C)$ satisfying
$|\widehat Q(C)-Q(C)|\leq\tau_\epsilon$.
Then
\begin{align*}
 f\text{ unsatisfiable}
 &\ \Longrightarrow\ 
 \widehat Q(C)<\beta_\epsilon+\tau_\epsilon<\frac14,\\
 f\text{ satisfiable}
 &\ \Longrightarrow\ 
 \widehat Q(C)>\frac12-\beta_\epsilon-\tau_\epsilon>\frac14.
\end{align*}
The output description has polynomial length because the construction
algorithm runs in polynomial time; the marginal query has polynomial
cost and constant requested accuracy.
Consequently, comparison with $1/4$ decides SAT in deterministic
polynomial time, which implies $\mathrm P=\mathrm{NP}$.
\end{proof}

\begin{restated}[Corollary~\ref{cor:no_universal_approximation} (No universal efficient approximation)]
Fix $0<\epsilon<1/8$.
Unless $\mathrm{P}=\mathrm{NP}$, no deterministic polynomial-time
algorithm can, for every target specified by an efficiently
evaluable unnormalized representation $\tilde p$, output an
MPO-BM $q_\rho$ with bond dimension $R=\mathrm{poly}(D)$ satisfying
$D_{\mathrm{KL}}(q_\rho\Vert p)<\epsilon$, where $p=\tilde p/Z$.
\end{restated}

\begin{proof}
For this computational corollary, we make the numerical representation
underlying tractable marginalization explicit. The local dimension $K$
is fixed, and MPO core entries are rational numbers encoded in binary.
As assumed in Section~\ref{sec:MPO_Born_Machine}, the required local basis
integrals are efficiently evaluable: additive accuracy $2^{-t}$ costs time
polynomial in $t$, the rational interval's encoding length, and the
basis description length if the basis is part of the output.
The basis description is otherwise fixed.
A polynomial-time construction necessarily outputs a polynomial number
of bits; the polynomial bond bound alone does not control entry precision.

These conventions give a concrete bit-model implementation of the
marginalization property used in Theorem~\ref{thm:kl-hardness}.

For a coordinate cylinder $C$ restricting coordinate $d\in S$ to the
rational interval $I_d$, set
$B_d=\int_{I_d}\phi(t)\phi(t)^\top\,dt$; set $B_d=I_K$ for $d\notin S$.
Orthonormality gives $0\preceq B_d\preceq I_K$, so every entry has
absolute value at most one.
Its probability is the chain contraction
\[
 Q_\rho(C)=\operatorname{tr}\!\left(\rho\bigotimes_{d=1}^D B_d\right),
\]
requiring $O(DK^2R^2)$ arithmetic operations once the $B_d$ are available.
This contraction can also be evaluated in polynomial bit complexity.
To see this without assuming numerical conditioning, suppose each core
entry has absolute value at most $2^{b_\rho}$, where $b_\rho$ is bounded
by the output encoding length.
Expanding the contraction gives at most $K^{2D}R^{D-1}$ terms.
If every entry of every $B_d$ is approximated within $\nu\leq1$, a
telescoping product bound shows that the resulting probability error is
at most
\[
 A_D\nu,
 \qquad
 A_D:=DK^{2D}R^{D-1}2^{Db_\rho+D-1}.
\]
Thus $\nu\leq\min\{1,\tau/A_D\}$ suffices for additive error $\tau$.
The required number of precision bits is polynomial in $D$, the core
encoding length, and $\log(1/\tau)$ because $\log A_D$ is polynomial in
these quantities.
Use dyadic approximations to the $B_d$ and clear the rational
denominators in each core before contracting. The logarithms of these
common denominators are bounded by the total core encoding length, and
all intermediate integer contractions have polynomial bit length.
This gives a polynomial-time marginal-query algorithm, including the
cost of precision and intermediate arithmetic.
The hardness theorem now applies.
\end{proof}

\paragraph{Scope of the conclusion.}
The reduction concerns deterministic construction from an unnormalized
target representation, not approximation from a normalized-density or
sampling oracle.
The constructed targets have separated supports and need not satisfy the
strict positivity and log-Sobolev assumptions of the positive results.
Thus the reduction neither contradicts those conditional guarantees nor
establishes Fisher-divergence hardness under their assumptions.
It rules out a universal construction; it does not establish hardness
under the additional promise that a specified feature space already
admits a polynomial-bond approximation.

\section{Representation and Finite-Query Guarantees}
This appendix proves the population approximation and finite-query results.
The proofs use three steps: approximate the population optimizer in trace
norm, control the error of the local estimator, and transfer this control
to the empirical state. Throughout, $H$ is the population matrix,
$\lambda_{\min}$ its smallest eigenvalue, and
$\Lambda=\lambda_{\max}(H)-\lambda_{\min}(H)$ its spectral width.

\subsection{Polynomial-bond approximation}
\label{sec:proof_model_complexity}

\begin{restated}[Theorem~\ref{thm:model_complexity} (Polynomial-bond approximation)]
Let $p$ be a strictly positive $C^1$ density satisfying $LSI(c)$ with
$c^{-1}\leq\operatorname{poly}(D)$, and let $\Phi$ be a $C^1$
orthonormal tensor-product feature map. Assume that the integrals defining
$H$ in~\eqref{eq:def_H} are finite. Suppose that $H$ is a one-dimensional
local Hamiltonian with fixed local dimension $K$, uniformly bounded
interaction range and strength, and a dimension-independent spectral gap
above its ground space. Then, for every $\epsilon>0$, there exists a
density operator $\rho_{\mathrm{mpo}}\in\mathrm{MPO}(R)$ with
$R\leq\operatorname{poly}(D,\epsilon^{-1})$ such that
\[
 D_{\mathrm{KL}}(q_{\mathrm{mpo}}\Vert p)
 \leq\epsilon+\frac{\lambda_{\min}}{2c},
 \qquad
 q_{\mathrm{mpo}}(x)=\Phi(x)^\top\rho_{\mathrm{mpo}}\Phi(x).
\]
\end{restated}

We first state the tensor-network approximation result used in the proof.
Here bounded interaction strength means that, after terms with the same
support are combined, the sum of their norms incident on any site is
bounded independently of the chain length. Together with bounded
interaction range, this gives $O(D)$ bounded local terms.

\begin{theorem}[MPO approximation of the maximally mixed ground state]
\label{thm:mpo_approximation_mmgs}
Consider a one-dimensional finite-range local Hamiltonian on $n$ sites
with fixed local dimension, uniformly bounded interaction strength, and
a spectral gap bounded below by a positive constant independent of $n$.
Let $\Omega$ be its maximally mixed ground state. For every
$\tau\in(0,1)$, there exists a positive semidefinite, unit-trace MPO
$\Psi_\tau$ such that
\[
 \|\Psi_\tau-\Omega\|_{S^1}\leq\tau,
 \qquad
 \operatorname{bond}(\Psi_\tau)\leq\operatorname{poly}(n/\tau).
\]
The polynomial may depend on the fixed local dimension, interaction
range and strength, and lower bound on the gap.
\end{theorem}

This is the fixed-parameter consequence of Corollary~2.9
of~\citet{arad2026area}. Their construction
approximates a purification and then traces out the ancillary system;
it therefore produces a density operator. Bounded finite-range
interactions reduce to the nearest-neighbor setting of that corollary by
grouping a constant number of adjacent sites. The grouped local dimension
and interaction norms remain bounded. Expanding the grouped tensors back
to the original sites changes the bond bound by only a constant factor.

The approximation can also be taken to be real when $\Omega$ is real,
as in our setting. Indeed, if $\Psi_\tau$ is complex, replace it by
\[
 \Psi_\tau^{\mathbb R}
 :=\frac{\Psi_\tau+\overline{\Psi_\tau}}{2}.
\]
This matrix is real, positive semidefinite, and has unit trace.
Conjugation preserves the trace norm and $\overline\Omega=\Omega$, so
its approximation error is at most $\tau$. Across every chain cut,
taking the real part increases the tensor unfolding rank by at most a
factor of two, hence it preserves the polynomial MPO bond bound.

\begin{proof}[Proof of Theorem~\ref{thm:model_complexity}]
Let $P$ be the projector onto the ground space of $H$, let
$r=\operatorname{tr}(P)$, and set $\rho^*=P/r$.
Then $\operatorname{tr}(\rho^*H)=\lambda_{\min}$.
Lemma~\ref{lem:energy_kl_certificate} gives the following bound for
every density operator $\rho$:
\begin{equation}
 D_{\mathrm{KL}}(q_\rho\Vert p)
 \leq\frac{\operatorname{tr}(\rho H)}{2c}.
 \label{eq:app_population_energy_certificate}
\end{equation}
Thus it suffices to approximate $\rho^*$ while increasing its objective
value by at most $2c\epsilon$.

If $\Lambda=0$, then $H=\lambda_{\min}I$ and
$\rho^*=I/K^D=(I_K/K)^{\otimes D}$ is already a density operator with
MPO bond one. The result follows from
\eqref{eq:app_population_energy_certificate}. Hence assume $\Lambda>0$.
Apply Theorem~\ref{thm:mpo_approximation_mmgs} with
\[
 \tau:=\min\left\{\frac12,\frac{2c\epsilon}{\Lambda}\right\}.
\]
It yields a density operator $\rho_{\mathrm{mpo}}$ satisfying
\begin{equation}
 \begin{aligned}
 \|\rho_{\mathrm{mpo}}-\rho^*\|_{S^1}
 &\leq\frac{2c\epsilon}{\Lambda},\\
 R&\leq\operatorname{poly}\!\left(
 D\max\left\{1,\frac{\Lambda}{c\epsilon}\right\}\right).
 \end{aligned}
 \label{eq:app_population_trace_approximation}
\end{equation}
Because both states have unit trace, the identity part of $H$ cancels:
\begin{equation}
 \begin{aligned}
 \operatorname{tr}(\rho_{\mathrm{mpo}}H)-\lambda_{\min}
 &=\operatorname{tr}\!\left[
 (\rho_{\mathrm{mpo}}-\rho^*)(H-\lambda_{\min}I)\right]\\
 &\leq\|\rho_{\mathrm{mpo}}-\rho^*\|_{S^1}
       \|H-\lambda_{\min}I\|_{S^\infty}\\
 &\leq 2c\epsilon.
 \end{aligned}
 \label{eq:app_trace_to_energy}
\end{equation}
The second line is trace--operator norm duality. Combining this estimate
with~\eqref{eq:app_population_energy_certificate} proves the KL bound.

Finally, the spectral-width bound~\eqref{eq:app_spectral_width_bound}
gives $\Lambda=O(D)$.
Since $c^{-1}\leq\operatorname{poly}(D)$, the bond bound in
\eqref{eq:app_population_trace_approximation} is polynomial in
$D$ and $\epsilon^{-1}$. Its constants can depend on the fixed
structural parameters, but not on $D$ or $\epsilon$.
\end{proof}

\subsection{Locality for path-graph Markov random fields}
\label{sec:proof_theorem_MRF}

\begin{restated}[Proposition~\ref{thm:MRF} (Locality for path-graph MRFs)]
Under the differentiability and integrability assumptions defining $H$,
a strictly positive path-graph MRF induces a one-dimensional local
Hamiltonian with interaction range at most three sites.
\end{restated}

\begin{proof}
Write the path-graph factorization as
\begin{equation}
 p(x)=\frac{1}{Z}
 \prod_{j=1}^D\psi_j(x_j)
 \prod_{j=1}^{D-1}\psi_{j,j+1}(x_j,x_{j+1}),
 \label{eq:appendix_path_MRF}
\end{equation}
where the node and edge potentials are positive and differentiable.
For an interior site $j$, differentiating the log density gives
\begin{equation}
 \begin{aligned}
 s_j(x)
 &=\partial_{x_j}\log\psi_j(x_j)
   +\partial_{x_j}\log\psi_{j-1,j}(x_{j-1},x_j)\\
 &\quad+\partial_{x_j}\log\psi_{j,j+1}(x_j,x_{j+1}).
 \end{aligned}
\end{equation}
At the endpoints, the missing edge term is omitted. Thus $s_j$ depends
only on $C_j:=\{j-1,j,j+1\}\cap[D]$.

Let $\ell_j=\min C_j$, $u_j=\max C_j$, and define
\[
 \Phi_{C_j}:=\bigotimes_{d\in C_j}\phi(x_d),
 \qquad
 b_j:=2\partial_{x_j}\Phi_{C_j}-s_j\Phi_{C_j}.
\]
These are vectors in $\mathbb R^{K^{|C_j|}}$; their dependence on
$x_{C_j}$ is suppressed here. In the notation of
\eqref{eq:local_gram_integrand}, $h_j=b_jb_j^\top$.
The Gram form of~\eqref{eq:def_H} is
\begin{equation}
 H=\sum_{j=1}^D\int
 w_j(x)w_j(x)^\top\,dx,
 \qquad
 w_j:=2\partial_{x_j}\Phi-s_j\Phi\in\mathbb R^{K^D}.
 \label{eq:app_full_gram_sum}
\end{equation}
Since the score component is local, its summand factors as
\[
 w_j=
 \left(\bigotimes_{d<\ell_j}\phi(x_d)\right)
 \otimes b_j\otimes
 \left(\bigotimes_{d>u_j}\phi(x_d)\right).
\]
Integrating each factor outside $C_j$ gives $I_K$ by orthonormality.
Consequently, with
\[
 A_j:=\int_{\mathbb R^{|C_j|}} h_j(x_{C_j})\,dx_{C_j}\succeq0,
\]
we obtain
\begin{equation}
 H=\sum_{j=1}^D
 I_K^{\otimes(\ell_j-1)}\otimes A_j
 \otimes I_K^{\otimes(D-u_j)}.
 \label{eq:app_path_local_decomposition}
\end{equation}
Each term acts on at most three consecutive sites, which proves the
claim, including the endpoint cases. The finite Gram integrals make the
entries absolutely integrable by Cauchy--Schwarz, so Fubini's theorem
justifies integrating the coordinates separately.
\end{proof}

This proposition establishes locality. Uniform bounds on interaction
strength, the target's LSI constant, and the gap are additional
conditions of Theorem~\ref{thm:model_complexity}.

\subsection{Concentration of the local estimator}
\label{sec:proof_concentraction}

\paragraph{Local access and sampling assumptions.}
The estimator requires each score component $s_j$ to be evaluable using
only $x_{C_j}$ for a known bounded-size interval $C_j$. This is an
additional access assumption: locality of the integrated matrix $H$
alone does not imply it. Since the features and their derivatives are
known, one such scalar score evaluation determines $h_j(x_{C_j})$.
For importance sampling, the proposal $\pi_j$ must be positive wherever
$h_j$ is nonzero. A uniform proposal below is a probability density on
a finite-volume box covering that integrand's support; no uniform
probability density exists on all of $\mathbb R^{|C_j|}$.
The finite-query theorem combines the population representation
assumptions with these access conditions and the bounded-support,
bounded-integrand, and independent-sampling conditions of the following
lemma. The population representation assumptions alone do not establish
the query bound.

\begin{restated}[Lemma~\ref{thm:concentration} (Concentration of the local estimator)]
Consider the local estimator in~\eqref{eq:local_H_estimator}.
Suppose that $|C_j|\leq C$ and $h_j$ is supported on
$\mathcal X_j=[-\ell/2,\ell/2]^{|C_j|}$, with
$\|h_j(x_{C_j})\|_{S^\infty}\leq J$ almost everywhere, where
$C,\ell,J$ are fixed constants. For each $j$, draw $m$ independent
uniform samples on $\mathcal X_j$, independently across terms, and use
\[
 \bar A_j=\frac{\ell^{|C_j|}}{m}
 \sum_{i=1}^m h_j(X_{j,i}).
\]
For every $\epsilon,\delta\in(0,1)$,
$\Pr\{\|\bar H-H\|_{S^\infty}\leq\epsilon\}\geq1-\delta$ whenever
\[
 m\geq C_0\frac{D^2}{\epsilon^2}
 \log\!\left(\frac{2D\,\operatorname{poly}(K)}{\delta}\right),
\]
where $C_0$ depends only on $C,\ell,J$. Thus a sufficient total local
score-query budget is
$B=Dm=O(D^3\epsilon^{-2}\log(2D\,\operatorname{poly}(K)/\delta))$.
The polynomial degree is fixed by the maximum local support size $C$.
\end{restated}

The factor $\operatorname{poly}(K)$ can be taken to be $K^C$,
an upper bound on the local matrix dimensions.
We use the self-adjoint form of matrix Bernstein
\citep[Theorem~6.1.1]{tropp2015introduction}: for independent centered
$d\times d$ self-adjoint matrices $Y_i$ with
$\|Y_i\|_{S^\infty}\leq L$ and
$\|\sum_i\mathbb E Y_i^2\|_{S^\infty}\leq v$, it gives
\begin{equation}
 \Pr\!\left\{\left\|\sum_iY_i\right\|_{S^\infty}\geq t\right\}
 \leq 2d\exp\!\left(-\frac{t^2}{2(v+Lt/3)}\right).
 \label{eq:app_matrix_bernstein}
\end{equation}

\begin{proof}[Proof of Lemma~\ref{thm:concentration}]
For $C_j=\{\ell_j,\ldots,u_j\}$, abbreviate the identity-tensor
embedding by
\[
 \mathcal E_j(A):=
 I_K^{\otimes(\ell_j-1)}\otimes A
 \otimes I_K^{\otimes(D-u_j)}.
\]
It preserves the operator norm. Thus
$H=\sum_j\mathcal E_j(A_j)$ and
$\bar H=\sum_j\mathcal E_j(\bar A_j)$, with $A_j=\int h_j$.

Let $d_j=K^{|C_j|}$, $V_j=\ell^{|C_j|}$, and
\[
 M_0:=J\max\{1,\ell^C\},\qquad L_0:=2M_0.
\]
The importance weight is the volume $V_j$ of the local integration
domain. Hence the centered local matrix
\[
 Y_{j,i}:=V_jh_j(X_{j,i})-A_j
\]
has expectation zero. Furthermore,
\[
 \|A_j\|_{S^\infty}\leq V_jJ\leq M_0,
 \qquad \|Y_{j,i}\|_{S^\infty}\leq L_0.
\]
Since $Y_{j,i}$ is self-adjoint,
$Y_{j,i}^2\preceq L_0^2I$, giving
\[
 \left\|\sum_{i=1}^m\mathbb E Y_{j,i}^2\right\|_{S^\infty}
 \leq mL_0^2.
\]
Its dimension is $d_j\leq K^C$, rather than $K^D$.
Applying~\eqref{eq:app_matrix_bernstein} to the sum of these $m$
matrices, at threshold $mt$, yields
\begin{equation}
 \Pr\{\|\bar A_j-A_j\|_{S^\infty}>t\}
 \leq2K^C\exp\!\left(
 -\frac{mt^2}{2(L_0^2+L_0t/3)}\right).
 \label{eq:app_one_block_bernstein}
\end{equation}
Choose $t=\epsilon/D$. A union bound over the $D$ terms bounds the
failure probability by
\[
 2DK^C\exp\!\left(
 -\frac{m\epsilon^2}
 {2D^2(L_0^2+L_0\epsilon/(3D))}\right).
\]
Because $\epsilon<1$ and $D\geq1$, the choice
\[
 C_0:=2(L_0^2+L_0/3)
\]
in the stated lower bound on $m$ makes this probability at most
$\delta$. The term linear in $D/\epsilon$ in the usual Bernstein
sample bound has therefore been absorbed into $C_0D^2/\epsilon^2$.
On the resulting event,
\[
 \|\bar H-H\|_{S^\infty}
 \leq\sum_{j=1}^D\|\bar A_j-A_j\|_{S^\infty}
 \leq\epsilon.
\]
Each local sample uses one scalar score evaluation, so multiplying the
per-term sample size by $D$ gives the asserted budget $B$.
\end{proof}

\subsection{Finite-query approximation and stability}
\label{sec:proof_query_complexity}

\begin{restated}[Theorem~\ref{thm:query_complexity} (Finite-query approximation and stability)]
Assume the conditions of Theorem~\ref{thm:model_complexity} and
Lemma~\ref{thm:concentration}. Let $P$ project onto the ground space
of $H$, let $r=\operatorname{tr}(P)$ be known, and set $\rho^*=P/r$.
Let $\gamma_D$ be the gap above this space and let $\bar\rho$ solve
\eqref{eq:EigenVI_extended_estimated}. For $\epsilon>0$ and
$\delta\in(0,1)$, define
\[
 \eta:=\min\left\{\frac12,\frac{\gamma_D}{4},
 \frac{\gamma_Dc\epsilon}{c\epsilon+2\Lambda}\right\}.
\]
If $\bar H$ uses $m$ samples per local term and
\[
 B=Dm\geq C_1\frac{D^3}{\eta^2}
 \log\!\left(\frac{2D\,\operatorname{poly}(K)}{\delta}\right),
\]
then, with probability at least $1-\delta$,
\[
 \|\bar\rho-\rho^*\|_{S^1}\leq\frac{c\epsilon}{\Lambda},
 \qquad
 D_{\mathrm{KL}}(q_{\bar\rho}\Vert p)
 \leq\frac{\lambda_{\min}}{2c}+\frac{\epsilon}{2}.
\]
Moreover, there exists a density operator
$\rho_{\mathrm{mpo}}\in\mathrm{MPO}(R)$, independent of the query
samples, with $R\leq\operatorname{poly}(D,\epsilon^{-1})$, such that,
on the same event,
\[
 \|\rho_{\mathrm{mpo}}-\bar\rho\|_{S^1}
 \leq\frac{2c\epsilon}{\Lambda},
 \qquad
 D_{\mathrm{KL}}(q_{\mathrm{mpo}}\Vert p)
 \leq\frac{\lambda_{\min}}{2c}+\frac{\epsilon}{2}.
\]
Here $C_1$ depends only on the fixed local parameters in
Lemma~\ref{thm:concentration}. The proximity assertions concern
the nontrivial case $\Lambda>0$.
\end{restated}

We first establish the perturbation bound for the normalized projector.
This is a ground-space version of the Davis--Kahan subspace bound
\citep{davis1970rotation}; the short proof below makes its dependence
on the multiplicity explicit.

\begin{lemma}[Stability of the maximally mixed ground state]
\label{thm:stability_maximally_mixed_ground_state}
Let $H\in\mathbb R^{N\times N}$ be symmetric with smallest eigenvalue $\lambda_{\min}$,
ground-space projector $P$ of rank $1\leq r<N$, and gap $\gamma>0$.
Let $\bar H=H+E$, where $E$ is symmetric and
$e:=\|E\|_{S^\infty}\leq\gamma/2$. If $Q$ projects onto an invariant
subspace spanned by $r$ lowest eigenvectors of $\bar H$, then
\[
 \left\|\frac{P}{r}-\frac{Q}{r}\right\|_{S^1}
 \leq\frac{2e}{\gamma-e}.
\]
If $e<\gamma/2$, this bottom-$r$ subspace is uniquely defined.
\end{lemma}

\begin{proof}
Choose an orthonormal basis matrix $U$ for the range of $P$ and an
orthonormal eigenbasis matrix $V$ for the range of $I-Q$.
Write $V^\top\bar H=\widehat L V^\top$, where $\widehat L$ is diagonal.
Weyl's eigenvalue inequality gives
\[
 \widehat L\succeq(\lambda_{\min}+\gamma-e)I.
\]
Since $HU=\lambda_{\min}U$, the eigenvalue equations imply
\[
 (\widehat L-\lambda_{\min}I)V^\top U=V^\top EU.
\]
The matrix on the left is invertible, so
\[
 \|(I-Q)P\|_{S^\infty}
 =\|V^\top U\|_{S^\infty}
 \leq\frac{e}{\gamma-e}.
\]
For equal-rank orthogonal projectors,
$\|P-Q\|_{S^\infty}=\|(I-Q)P\|_{S^\infty}$; both equal the largest
sine of their principal angles. Also $\operatorname{rank}(P-Q)\leq2r$.
Consequently,
\[
 \left\|\frac{P-Q}{r}\right\|_{S^1}
 \leq2\|P-Q\|_{S^\infty}
 \leq\frac{2e}{\gamma-e}.
\]
Finally, Weyl's inequality gives
$\lambda_{r+1}(\bar H)-\lambda_r(\bar H)\geq\gamma-2e$, which is
positive when $e<\gamma/2$.
\end{proof}

\begin{proof}[Proof of Theorem~\ref{thm:query_complexity}]
The variational characterization of the sum of the $r$ smallest
eigenvalues shows that every optimizer of
\eqref{eq:EigenVI_extended_estimated} is $\bar\rho=Q/r$, with $Q$
the bottom-$r$ spectral projector whenever that subspace is separated.
We compare this estimator with $\rho^*=P/r$ and a fixed population
MPO approximation.

First, because $\eta\leq1/2$, Lemma~\ref{thm:concentration} applies
with accuracy $\eta$. The assumed budget, with $C_1\geq C_0$,
ensures the event
\[
 \mathcal A:=\{\|\bar H-H\|_{S^\infty}\leq\eta\}
\]
has probability at least $1-\delta$. On this event,
$\eta\leq\gamma_D/4$ ensures that the bottom-$r$ empirical subspace
is separated. Lemma~\ref{thm:stability_maximally_mixed_ground_state}
then gives
\[
 \|\bar\rho-\rho^*\|_{S^1}
 \leq\frac{2\eta}{\gamma_D-\eta}
 \leq\frac{c\epsilon}{\Lambda}.
\]
For the last inequality, use the increasing function
$t\mapsto2t/(\gamma_D-t)$ and
$\eta\leq\gamma_Dc\epsilon/(c\epsilon+2\Lambda)<\gamma_D$.
The trace-zero shift in~\eqref{eq:app_trace_to_energy}, now applied
to $\bar\rho$, yields
\[
 \operatorname{tr}(\bar\rho H)-\lambda_{\min}
 \leq\Lambda\|\bar\rho-\rho^*\|_{S^1}
 \leq c\epsilon.
\]
Lemma~\ref{lem:energy_kl_certificate} proves the asserted KL bound
for $q_{\bar\rho}$.

Next, apply Theorem~\ref{thm:mpo_approximation_mmgs} to the
\emph{population} state $\rho^*$ with
\[
 \tau:=\min\left\{\frac12,\frac{c\epsilon}{\Lambda}\right\}.
\]
Fix one resulting density operator $\rho_{\mathrm{mpo}}$ before
drawing any query samples. It satisfies
\[
 \|\rho_{\mathrm{mpo}}-\rho^*\|_{S^1}
 \leq\frac{c\epsilon}{\Lambda},
 \qquad
 R\leq\operatorname{poly}\!\left(
 D\max\left\{1,\frac{\Lambda}{c\epsilon}\right\}\right).
\]
Its population energy is at most $\lambda_{\min}+c\epsilon$, so
Lemma~\ref{lem:energy_kl_certificate} gives
\[
 D_{\mathrm{KL}}(q_{\mathrm{mpo}}\Vert p)
 \leq\frac{\lambda_{\min}}{2c}+\frac{\epsilon}{2}.
\]
This KL bound holds independently of the query samples. On the event
$\mathcal A$, the triangle inequality additionally gives
\[
 \|\rho_{\mathrm{mpo}}-\bar\rho\|_{S^1}
 \leq\|\rho_{\mathrm{mpo}}-\rho^*\|_{S^1}
      +\|\bar\rho-\rho^*\|_{S^1}
 \leq\frac{2c\epsilon}{\Lambda}.
\]
Finally, \eqref{eq:app_spectral_width_bound} gives $\Lambda=O(D)$.
Together with $c^{-1}\leq\operatorname{poly}(D)$, this yields the
claimed polynomial bond bound. The gap assumption also implies
\[
 \eta^{-1}
 =\max\left\{2,\frac{4}{\gamma_D},
 \frac{1}{\gamma_D}\left(1+\frac{2\Lambda}{c\epsilon}\right)\right\}
 \leq\operatorname{poly}(D,\epsilon^{-1}),
\]
so the displayed sufficient query budget is polynomial as well.
\end{proof}

The MPO in this proof is a population approximator whose existence is
established before sampling. Proximity to it does not give a procedure
for constructing it from $\bar H$; the theorem makes no polynomial-time
optimization or compression claim.

\subsection{A comparison with uniform global estimation}

For comparison with the local estimator, we give a concentration bound
for the global importance-sampling estimator in~\eqref{eq:H_estimation}
under a uniform proposal. Its assumptions concern the full integrand
$H(x)$, whose matrix dimension is $K^D$.

For a general importance proposal $\pi$, unbiased estimation of the full
matrix requires $\pi(x)>0$ wherever $H(x)\ne0$, up to null sets.
Covering only the locations where the target score is nonzero is not
sufficient, because feature-derivative terms may remain nonzero where the score vanishes.

\begin{lemma}[Concentration of the uniform global estimator]
\label{thm:concentration_of_global_H}
Let $H(x)$ be a measurable symmetric $K^D\times K^D$ matrix-valued function
supported on $\mathcal X=[-\ell/2,\ell/2]^D$, with
\[
 \|H(x)\|_{S^\infty}\leq D J_{\mathrm g}
\]
almost everywhere, where $J_{\mathrm g}>0$ is a bound for the full
integrand and may depend on $D$.
Define $H:=\int_{\mathcal X}H(x)\,dx$ and
$V:=\ell^D$. For $m$ independent uniform samples on $\mathcal X$,
let $\bar H_{\mathrm g}:=Vm^{-1}\sum_{i=1}^m H(X_i)$.
Then, for $\epsilon>0$ and $\delta\in(0,1)$,
$\Pr\{\|\bar H_{\mathrm g}-H\|_{S^\infty}\leq\epsilon\}\geq1-\delta$ whenever
\begin{equation}
 \begin{aligned}
 m\geq&\left(\frac{8J_{\mathrm g}^2V^2D^2}{\epsilon^2}
       +\frac{4J_{\mathrm g}VD}{3\epsilon}\right)\\
 &\qquad\cdot\left(D\log K+\log\frac{2}{\delta}\right).
 \end{aligned}
 \label{eq:app_global_bernstein_budget}
\end{equation}
\end{lemma}

\begin{proof}
The estimator is unbiased because the proposal density is $1/V$ on
$\mathcal X$. The matrices
$Y_i:=VH(X_i)-H$ are independent, centered, and symmetric, with
\[
 \|Y_i\|_{S^\infty}\leq2VDJ_{\mathrm g}=:L_{\mathrm g},
 \qquad
 \left\|\sum_{i=1}^m\mathbb E Y_i^2\right\|_{S^\infty}
 \leq mL_{\mathrm g}^2.
\]
Applying~\eqref{eq:app_matrix_bernstein} in dimension $K^D$ gives
\[
 \Pr\{\|\bar H_{\mathrm g}-H\|_{S^\infty}>\epsilon\}
 \leq2K^D\exp\!\left(
 -\frac{m\epsilon^2}{2(L_{\mathrm g}^2+L_{\mathrm g}\epsilon/3)}
 \right).
\]
Substituting $L_{\mathrm g}=2VDJ_{\mathrm g}$ and requiring the
right-hand side to be at most $\delta$ gives
\eqref{eq:app_global_bernstein_budget}.
\end{proof}

To relate this assumption to local score structure, the factorization
in~\eqref{eq:app_full_gram_sum} gives
\[
 \|w_j(x)w_j(x)^\top\|_{S^\infty}
 =\|h_j(x_{C_j})\|_{S^\infty}
   \prod_{d\notin C_j}\|\phi(x_d)\|_2^2.
\]
Thus the bound on $h_j$ in Lemma~\ref{thm:concentration} does not by
itself give a dimension-independent $J_{\mathrm g}$: the feature
factors outside $C_j$ must also be controlled. Global compact support
is likewise an additional assumption. Applying a uniform proposal on
a truncated box to a noncompact integrand estimates the truncated
matrix; the truncation error must then be bounded separately.

With fixed $J_{\mathrm g}$, the quadratic term in
\eqref{eq:app_global_bernstein_budget} contains
$V^2=\ell^{2D}$, which grows exponentially when $\ell>1$.
Local estimation instead uses bounded-dimensional volumes and matrices.
Both results are sufficient norm-based concentration bounds, so this
comparison does not establish an exponential lower bound for global
estimation. The relevant scale is the product of proposal volume and
integrand norm; a change of coordinates or a different proposal must
be assessed through both quantities.
Finally, $m$ here counts full-point samples. When constructing $H(X_i)$
requires all $D$ score components, the corresponding scalar
score-query budget is $Dm$, consistent with the counting convention
for the local estimator.

\section{Numerical Experiments}
\label{sec:experiments}
\label{sec:app_experimental_details}

We examine the representation and finite-query guarantees through three
questions. How do the feature dimension and MPO bond affect approximation
accuracy? How accurately do local queries estimate the Hamiltonian and its
ground state? How does local estimation affect the resulting density at a
matched score-query budget?
Sections~\ref{sec:exp_representation}--\ref{sec:exp_matched} address these
questions using a Gaussian path MRF and a non-Gaussian $\phi^4$ chain, both
with analytically available LSI constants. These experiments illustrate the
mechanisms in Theorems~\ref{thm:model_complexity}
and~\ref{thm:query_complexity}; the finite dimensions and ambient-space
eigensolvers do not establish asymptotic rates or polynomial-time training.

Throughout this appendix, $c_D$ is an admissible target LSI constant,
$\gamma_D$ is the population gap above the numerical ground space, and
$\Lambda_D=\lambda_{\max}(H)-\lambda_{\min}(H)$ is the spectral width.
The subscript records their dependence on dimension; these are the same
quantities denoted by $c$, $\gamma_D$, and $\Lambda$ in the theorems.

\subsection{Common targets, features, and metrics}
\label{sec:app_core_protocol}

\paragraph{Targets and local scores.}
Our benchmarks are a Gaussian path MRF and a non-Gaussian
$\phi^4$ chain, defined by
\begin{equation}
 p_D(x)=Z_D^{-1}\exp[-V_D(x)],\qquad
 V_D(x)=\frac12 x^\top P_Dx+\frac{\beta}{4}\sum_{j=1}^{D}x_j^4,
 \label{eq:experimental_targets}
\end{equation}
where $P_D$ has diagonal entries one, adjacent entries $0.3$, and all other
entries zero.  We use $\beta=0$ for the Gaussian and $\beta=0.25$ for
$\phi^4$.  Their score components are
\[
 s_j(x)=-x_j-\beta x_j^3-0.3\mathbf1_{j>1}x_{j-1}
                          -0.3\mathbf1_{j<D}x_{j+1}.
\]
Thus each score component depends only on $C_j=\{j-1,j,j+1\}\cap[D]$.
Since $\nabla^2 V_D=P_D+3\beta\operatorname{diag}(x_j^2)\succeq c_DI$,
the Bakry--\'{E}mery criterion~\citep{bakry1985diffusions} gives
\begin{equation}
 c_D=1-0.6\cos\!\left(\frac{\pi}{D+1}\right)\geq0.4.
 \label{eq:experimental_lsi_constant}
\end{equation}
This is the optimal Gaussian LSI constant and an admissible, not necessarily
optimal, constant for $\phi^4$.  This target
property is distinct from the gap of the projected Hamiltonian, which is
measured numerically. All tested population Hamiltonians have a unique
numerical ground state.
Thus the experiments address the $r=1$ case of the theory and do not test
its extension to degenerate ground spaces.

\paragraph{Features and valid MPO construction.}
Following the orthogonal-expansion formulation of
EigenVI~\citep{cai2024eigenvi}, for $n=0,\ldots,K-1$ we use
\[
 \phi_n(x)=\frac{\operatorname{He}_n(x)e^{-x^2/4}}
                  {\sqrt{n!\sqrt{2\pi}}},\qquad
 \Phi(x)=\bigotimes_{j=1}^D\phi(x_j),
\]
where $\operatorname{He}_n$ is the probabilists' Hermite polynomial.
Here $\phi(x)=(\phi_0(x),\ldots,\phi_{K-1}(x))^\top$.
The scale, centering, and coordinate ordering remain fixed across estimators
and basis sizes.  For a normalized real MPS amplitude $\theta$ of bond at most
$\chi$, the model is
$q_\theta(x)=(\Phi(x)^\top\theta)^2$, and
$\rho_\theta=\theta\theta^\top$ is a PSD, unit-trace MPO.
Tensor-train singular value decomposition (TT-SVD) is applied
to the amplitude, followed by normalization, rather than to an arbitrary
density matrix; this preserves a valid density.
We report the constructive MPO bond $R=\chi_{\rm eff}^2$, where
$\chi_{\rm eff}\leq\chi$ is the largest amplitude bond after TT-SVD.
This upper bound need not equal the minimal MPO bond.

\paragraph{Population operator and local blocks.}
For $C_j=\{\ell_j,\ldots,u_j\}$, define the local feature vector
$\Phi_{C_j}(x_{C_j})=\bigotimes_{d\in C_j}\phi(x_d)$ and the local Gram
integrand
\[
 b_j(x_{C_j})=2\partial_{x_j}\Phi_{C_j}(x_{C_j})
                 -s_j(x_{C_j})\Phi_{C_j}(x_{C_j}),
 \qquad h_j(x_{C_j})=b_j(x_{C_j})b_j(x_{C_j})^\top.
\]
The population local block and its full-space assembly are
\begin{equation}
 A_j=\int_{\mathbb R^{|C_j|}}h_j(x_{C_j})\,dx_{C_j},
 \qquad
 H=\sum_{j=1}^{D}I_K^{\otimes(\ell_j-1)}\otimes A_j
                      \otimes I_K^{\otimes(D-u_j)}.
 \label{eq:experimental_population_blocks}
\end{equation}
The identity factors follow from orthonormality outside $C_j$; an empty
factor equals one.  Thus reference integration is at most three-dimensional
for these targets, although the assembled matrix acts on $K^D$ features.

\paragraph{Metrics and numerical uncertainty.}
Reverse KL is the primary density metric because it is controlled by the
approximation theorems.  We evaluate
$D_{\rm KL}(q\Vert p)=\mathbb E_q[\log q-\log p]$ and
$D_F(q\Vert p)=\mathbb E_q\|\nabla\log q-s_p\|_2^2$, where
$s_p=\nabla\log p$.  For pure states, we compute Fisher divergence from the amplitude
$\psi=\theta^\top\Phi$ through
$\int\|2\nabla\psi-s_p\psi\|^2dx$. This equals $\theta^\top H\theta$
and avoids division by $\psi$ at its zeros. For a mixed state,
$\operatorname{tr}(\rho H)$ instead upper-bounds the Fisher divergence
(Appendix~\ref{sec:proof_mixed_fisher}).  All reverse-KL calculations
include the target normalizer.  For $\phi^4$, it is obtained by a
one-dimensional transfer contraction along the path; the Gaussian formula
provides an independent normalization control.
For every normalized state, write $E_\rho=\operatorname{tr}(\rho H)$.
The quantities measured below follow the certificate
\begin{equation}
 D_{\rm KL}(q_\rho\Vert p)
 \leq\frac{D_F(q_\rho\Vert p)}{2c_D}
 \leq\frac{E_\rho}{2c_D}.
 \label{eq:experimental_kl_certificate}
\end{equation}
The second inequality is an equality for the pure states used here.
Variability across fitted models is reported as the sample
standard deviation across seeds; uncertainty from sampling a fixed model is
reported as a Monte Carlo standard error (SE).  Quadrature refinement is
reported separately and is a numerical diagnostic rather than a certified
integration-error bound.

\subsection{Population accuracy versus feature dimension and MPO bond}
\label{sec:exp_representation}
\label{sec:app_basis_bond}

\paragraph{Question and construction.}
We first remove score-query noise by constructing $H$ with deterministic
local quadrature. We compress the normalized ground-state amplitude by
TT-SVD to an
MPS $\theta_\chi$ of bond at most $\chi$, renormalize it, and form
$\rho_\chi=\theta_\chi\theta_\chi^\top$.
This gives a PSD, unit-trace MPO with constructive bond bound $R\leq\chi^2$.
At fixed $K$, its energy certificate decomposes as
\begin{equation}
 \frac{E_{\rho_\chi}}{2c_D}
 =\underbrace{\frac{\lambda_{\min}(H)}{2c_D}}_{\text{irreducible term}}
 +\underbrace{\frac{E_{\rho_\chi}-\lambda_{\min}(H)}{2c_D}}
             _{\text{compression excess}}.
 \label{eq:experimental_error_decomposition}
\end{equation}
Increasing the bond can reduce the second term but cannot change the first
for a fixed target and feature space.  The irreducible term belongs to this
upper bound; it is not a lower bound on the achievable reverse KL.

\begin{figure}[t]
\centering
\includegraphics[width=\linewidth]{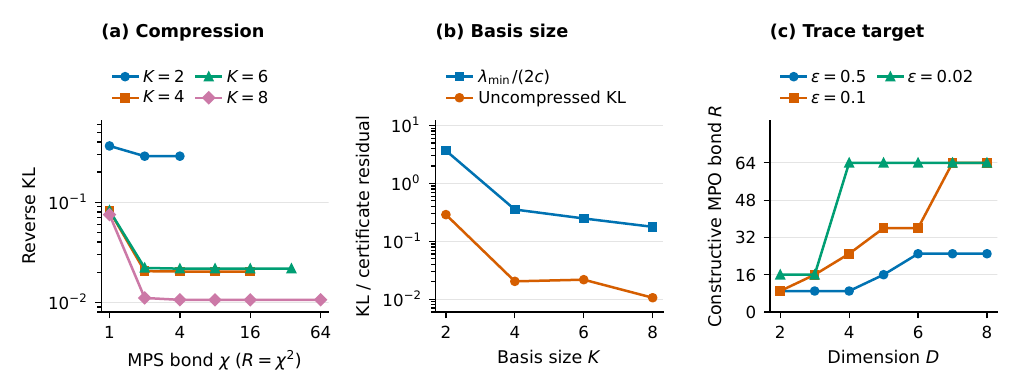}
\caption{Population approximation for the $\phi^4$ chain.
(a) Reverse KL versus amplitude MPS bond $\chi$ at $D=4$, with constructive
MPO bond $R\leq\chi^2$.
(b) The irreducible term $\lambda_{\min}/(2c_D)$ and reverse KL of the
uncompressed population ground state as $K$ changes.
Panels (a,b) use deterministic quadrature.
(c) Smallest tested constructive MPO bond satisfying
$\|\rho_\chi-\rho^*\|_{S^1}\leq2c_D\epsilon/\Lambda_D$ at fixed $K=4$.
These are sufficient finite-system constructions, not minimal-bond estimates
or an asymptotic scaling result.}
\label{fig:phi4_certified_lsi}
\end{figure}

\paragraph{Feature and bond sweep.}

For both targets, we vary $D\in\{2,4\}$, $K\in\{2,4,6,8\}$, and the
amplitude bond cap $\chi\in\{1,2,4,8,16\}$, also evaluating the
uncompressed ground state. Caps that give the same exact representation
are counted once. Local Hamiltonian terms use tensor-product Gauss--Hermite quadrature of order 12,
checked at order 16. Because the scores are polynomial, order at least
$K+3$ integrates each local Gram term exactly up to floating-point error. The largest sum of local Frobenius differences is
$7.78\times10^{-11}$, and the maximum ground-eigenpair residual is
$1.50\times10^{-12}$.

Reverse KL is not polynomial in the features.  We therefore compare orders
24 and 32 and adaptively increase the evaluation order through 48, 64, and
96 until the successive difference is below
$\max\{10^{-8},10^{-3}|D_{\rm KL}|\}$.  Every evaluated state meets this criterion. The maximum normalization
error is $1.33\times10^{-15}$, and the largest
discrepancy between independently integrated Fisher divergence and
Hamiltonian energy is $1.29\times10^{-14}$.  The target transfer integral is
checked by increasing order 160 on $[-8,8]$ to order 220 on $[-10,10]$.

\begin{table}[tbp]
\centering
\caption{Population feature-dimension sweep at $D=4$.  The irreducible term $\lambda_{\min}/(2c_D)$
belongs to the upper bound.  KL is evaluated for the uncompressed
Fisher-optimal ground state by deterministic quadrature; it need not decrease
monotonically as the Fisher objective decreases.}
\label{tab:basis_sweep}
\begin{tabular}{crrrr}
\toprule
& \multicolumn{2}{c}{Gaussian} & \multicolumn{2}{c}{$\phi^4$}\\
$K$ & Irreducible term & Reverse KL & Irreducible term & Reverse KL\\
\midrule
2 & $6.464\times10^{-2}$ & $1.757\times10^{-2}$ & 3.6811 & 0.28851\\
4 & $1.128\times10^{-3}$ & $1.723\times10^{-4}$ & 0.3534 & 0.02031\\
6 & $1.562\times10^{-5}$ & $1.578\times10^{-6}$ & 0.2478 & 0.02170\\
8 & $1.985\times10^{-7}$ & $1.570\times10^{-8}$ & 0.1776 & 0.01058\\
\bottomrule
\end{tabular}
\end{table}

Figure~\ref{fig:phi4_certified_lsi}(a,b) and Table~\ref{tab:basis_sweep}
show why bond and feature dimension must be varied separately.
At $D=4,K=8,\chi=4$ ($R\leq16$), the $\phi^4$ compression excess is below
$4.8\times10^{-7}$, whereas the irreducible term is $0.178$ and reverse KL
is $0.0106$.  Increasing $K$ from two to eight reduces the irreducible term
from $3.68$ to $0.178$ for $\phi^4$ and from $0.0646$ to
$1.99\times10^{-7}$ for the Gaussian.
Actual reverse KL need not decrease monotonically: the $\phi^4$ value rises
from $0.0203$ at $K=4$ to $0.0217$ at $K=6$ despite a lower ground energy.
Thus the optimized Fisher objective, its reverse-KL upper bound, and actual
reverse KL describe related but distinct properties of the approximation.

\paragraph{Fixed-feature dimension study.}
To examine dimension dependence at the fixed local dimension assumed by
Theorem~\ref{thm:model_complexity}, we use $K=4$, $D=2,\ldots,8$, and
$\epsilon\in\{0.5,0.2,0.1,0.05,0.02\}$.  Local terms are integrated at
Gauss--Hermite order 40. The amplitude bond caps for $\phi^4$ and the
Gaussian, respectively, are
\[
\begin{aligned}
\mathcal B_{\phi^4}
 &=\{1,2,3,4,5,6,8,10,12,16,24,32,48,64,96,128,192,256\},\\
\mathcal B_{\mathrm G}
 &=\{1,\ldots,16,24,32,48,64,96,128,192,256\}.
\end{aligned}
\]
At each dimension, caps exceeding the exact amplitude bound
$K^{\lfloor D/2\rfloor}$ are replaced by that bound and duplicates are
removed. The ground and extremal eigenpairs are computed with the symmetric
ARPACK solver at tolerance $10^{-9}$, with at most 5,000 iterations.
For $\phi^4$, orders 8, 12, 20, and 40 give maximum
local operator-norm difference $3.07\times10^{-12}$ and maximum gap
difference $5.68\times10^{-13}$.  The largest local-block norm is $111.884$;
the largest reported eigenpair residual, including extremal eigenvalues used
for $\Lambda_D$, is $5.39\times10^{-7}$. The transfer-normalizer check,
from order 200 on $[-8,8]$
to order 260 on $[-10,10]$, changes $\log Z_D$ by at most
$5.87\times10^{-14}$.

For each accuracy target we choose the smallest tested cap satisfying
$\|\rho_\chi-\rho^*\|_{S^1}\leq2c_D\epsilon/\Lambda_D$.
This is the sufficient trace-norm tolerance used in the proof of
Theorem~\ref{thm:model_complexity}. For normalized pure states $u,v$, we
compute trace distance from
$\|uu^\top-vv^\top\|_{S^1}=2\sqrt{1-|u^\top v|^2}$.
Across the $\phi^4$ dimension--accuracy grid, the maximum
$(E_{\rho_\chi}-\lambda_{\min})/(2c_D\epsilon)$ is
$6.38\times10^{-5}$.  Table~\ref{tab:phi4_certified_metrics} gives a
representative accuracy level.  Here reverse KL uses $10^5$ sequential
Born samples, evaluated on a 481-point conditional grid over $[-8,8]$,
using base random seed 43.
The uncertainties are within-model Monte Carlo standard errors; they are
not variation across training seeds or certified bounds on discretization
error.  This protocol differs from the deterministic $D\leq4$ basis sweep.

\begin{table}[tbp]
\centering
\caption{Non-Gaussian population construction at $K=4$, $\epsilon=0.1$.
The LSI constant is certified from the target; the gap $\gamma_D$ is a
finite-dimensional numerical diagnostic. $E_\chi:=E_{\rho_\chi}$ equals
pure-state Fisher divergence.}
\label{tab:phi4_certified_metrics}
\setlength{\tabcolsep}{5pt}
\begin{tabular}{rrrrrrr}
\toprule
$D$ & $c_D$ & $\gamma_D$ & $\lambda_{\min}/(2c_D)$ & $R$ & $E_\chi$
& Reverse KL (mean $\pm$ SE)\\
\midrule
2 & 0.7000 & 5.7407 & 0.1232 & 9 & 0.1725 & $0.01008\pm0.00089$\\
4 & 0.5146 & 5.1248 & 0.3534 & 25 & 0.3637 & $0.02034\pm0.00121$\\
6 & 0.4594 & 4.9428 & 0.6040 & 36 & 0.5550 & $0.03197\pm0.00144$\\
8 & 0.4362 & 4.8651 & 0.8554 & 64 & 0.7462 & $0.04030\pm0.00134$\\
\bottomrule
\end{tabular}
\end{table}

The smallest tested constructive $\phi^4$ bonds meeting the trace tolerance
range from $9$ to $64$ over all accuracy levels
(Figure~\ref{fig:phi4_certified_lsi}(c)); the observed population gaps range
from $4.865$ to $5.741$.
As a diagnostic of Proposition~\ref{prop:spectral_gap_sufficient}, substituting
the computed energies into $4c_D-\lambda_{\min}(H)$ gives positive values
throughout this grid, from $2.627$ at $D=2$ to $0.999$ at $D=8$.
They support the sufficient gap criterion at these dimensions; they do not
establish a dimension-independent bound.

The Gaussian control uses the same construction and exact Gaussian LSI
constant.  Its observed $K=4$ gaps decrease from $2.8025$ at $D=2$ to
$1.7766$ at $D=8$.  The $K$ sweep and the fixed-$K$ dimension sweep answer
different questions: the theorem fixes the local dimension and does not give
a joint polynomial guarantee in arbitrarily increasing $D$ and $K$.

\subsection{Local-query concentration and state stability}
\label{sec:exp_stability}
\label{sec:app_query_protocol}

\paragraph{Separating sampling error from domain truncation.}
We use the Gaussian path MRF ($\beta=0$) in this study.
Its Hermite-feature integrands have full support, whereas
Lemma~\ref{thm:concentration} assumes bounded-support integrands for uniform
sampling.  We therefore apply compact-domain sampling to the truncated
local integrals
\begin{equation}
\begin{aligned}
 A_{j,L}&:=\int_{[-L,L]^{|C_j|}}h_j(x_{C_j})\,dx_{C_j},\\
 \bar A_{j,L}&:=\frac{(2L)^{|C_j|}}{m}
              \sum_{i=1}^{m}h_j(X_{j,i}),
 \qquad X_{j,i}\stackrel{\rm iid}{\sim}
       \operatorname{Unif}([-L,L]^{|C_j|}),\\
 H_L&:=\sum_{j=1}^{D}I_K^{\otimes(\ell_j-1)}\otimes A_{j,L}
                         \otimes I_K^{\otimes(D-u_j)},
\end{aligned}
\label{eq:experimental_compact_estimator}
\end{equation}
where $L=5$ and $C_j=\{\ell_j,\ldots,u_j\}$ is the local score support.
The estimator $\bar H_L$ replaces $A_{j,L}$ by $\bar A_{j,L}$ in this sum.
It is unbiased for $H_L$, not for the full-space $H$.
Its blockwise error certificate is
\begin{equation}
 U:=\sum_{j=1}^{D}\|\bar A_{j,L}-A_{j,L}\|_{S^\infty},
 \qquad \|\bar H_L-H_L\|_{S^\infty}\leq U.
 \label{eq:experimental_local_certificate}
\end{equation}
The local block notation $A_j$ and $h_j$ agrees with
\eqref{eq:local_H_estimator}; identity factors embed each block in the full
feature space.

\begin{figure}[t]
\centering
\includegraphics[width=\linewidth]{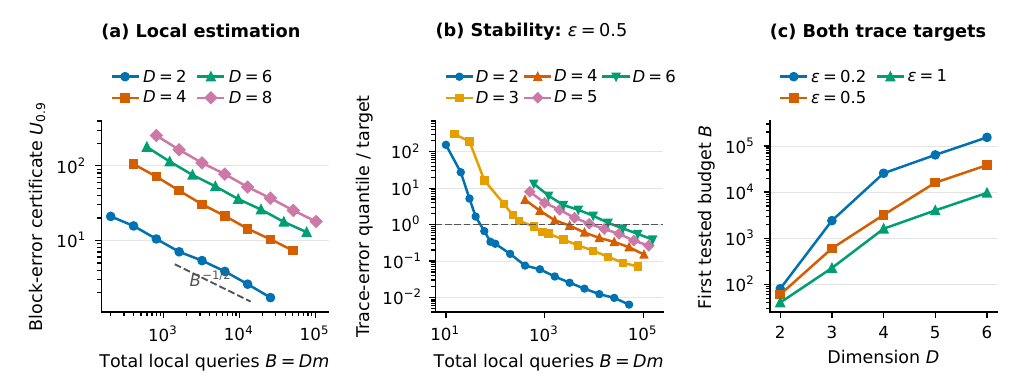}
\caption{Local estimation and state stability for the Gaussian path MRF at
$K=4$.  (a) Empirical 90th percentile of $U$, with 100 repetitions per point;
the dashed line is a $B^{-1/2}$ guide.
(b) Empirical 90th percentile of the trace error to the full-space population
state, divided by $c_D\epsilon/\Lambda_D$, at $\epsilon=0.5$ with 50
repetitions per point.
(c) First tested budget at which both trace-error quantiles meet their targets.
Panels (b,c) include domain-truncation effects; panel (c) also measures
proximity to a fixed population MPO of constructive bond $R=16$.
The crossings are descriptive statistics from the same repetitions used for
selection, not high-probability query guarantees.}
\label{fig:finite_query_theory_audit}
\end{figure}

\paragraph{Operator concentration.}

Sampling a local score component once supplies one local Gram integrand;
we count this as one query.  Boundary terms have support size two and
interior terms have support size three, so each importance weight uses its
own volume $10^{|C_j|}$ rather than the full-dimensional volume.
We test $K=4$, $D=2,\ldots,8$, and
$m\in\{100,200,400,800,1600,3200,6400,12800\}$ samples per block,
with 100 seeds (43--142) at each total local score-query budget $B=Dm$.
Deterministic reference
blocks use Gauss--Legendre order 24; increasing to 40 changes a block by at
most $4.35\times10^{-9}$ in operator norm. The empirical 90th percentiles
of $U$ have fitted log--log slopes between
$-0.554$ and $-0.512$ across dimensions, using all eight budgets with
$R^2\geq0.997$ (Figure~\ref{fig:finite_query_theory_audit}(a)).
This behavior is consistent with the $B^{-1/2}$ sampling-error dependence
in Lemma~\ref{thm:concentration}.
These fits describe budget dependence at fixed dimension; they do not
estimate the worst-case dimension exponent of its sufficient query bound.

\paragraph{State stability and MPO proximity.}
We next compare the estimated ground states with the full-space population
reference obtained by order-40 Gauss--Hermite quadrature. We use $K=4$,
$D=2,\ldots,6$, 50 seeds (43--92), and
$m\in\{100,200,400,800,1600,3200,6400,12800,25600\}$.
Additional budgets resolve the crossings in low dimensions:
$m\in\{5,10,15,20,30,40,50\}$ for $D=2$ and
$m\in\{5,10,20,50,75,300\}$ for $D=3$.
A fixed population MPO of constructive bond $R=16$ is used at every budget;
it is independent of the query samples.
The empirical eigensolver uses tolerance $10^{-8}$ and at most 3,000
iterations, and is initialized with the population ground vector
to isolate the response to operator perturbation. This reference-assisted
initialization is used only in this stability study; the learned-model
comparison in Section~\ref{sec:exp_matched} constructs each empirical model
without population eigenvectors.

Let $\bar\rho$ be the empirical ground projector of $\bar H_L$, let
$\rho^*$ be the full-space population ground projector, and let
$\rho_{\mathrm{mpo}}$ be the fixed compression of $\rho^*$.
For $\epsilon\in\{0.2,0.5,1\}$, we select the smallest tested $B$ at which
both marginal empirical 90th-percentile conditions hold:
\begin{equation}
\begin{aligned}
 q_{0.9}\!\left(\|\bar\rho-\rho^*\|_{S^1}\right)
      &\leq c_D\epsilon/\Lambda_D,\\
 q_{0.9}\!\left(\|\rho_{\mathrm{mpo}}-\bar\rho\|_{S^1}\right)
      &\leq 2c_D\epsilon/\Lambda_D.
\end{aligned}
\label{eq:experimental_trace_targets}
\end{equation}
Here $q_{0.9}$ denotes an empirical 90th percentile, using linear
interpolation between order statistics. Both targets are reached for every
dimension--accuracy pair (Figure~\ref{fig:finite_query_theory_audit}(b,c)).
The same 50 repetitions select and evaluate each crossing; there is no
held-out budget validation. The two marginal quantiles do not imply joint
90\% coverage. The observed fraction meeting both inequalities at the
selected budget ranges from $0.92$ to $1$.

The sufficient operator threshold is
\[
\eta_D=\min\!\left\{\frac12,\frac{\gamma_D}{4},
             \frac{\gamma_Dc_D\epsilon}{c_D\epsilon+2\Lambda_D}\right\}.
\]
At the selected crossings, the empirical 90th percentile of $U/\eta_D$
ranges from $1.524\times10^3$ to $1.001\times10^4$. Thus the blockwise
operator bound does not certify these budgets, even before adding domain
truncation. Writing
\[
 T_L:=\sum_j\|A_j-A_{j,L}\|_{S^\infty},\qquad
 \|\bar H_L-H\|_{S^\infty}\leq U+T_L,
\]
we find $T_L\leq0.2663$ over the tested dimensions, whereas the largest
trace discrepancy between the ground projectors of $H_L$ and $H$ is only
$3.26\times10^{-5}$. Small state error can therefore coexist with a much
larger operator-error bound. These results illustrate stability and
proximity to a fixed population MPO; they neither verify the theorem's
sufficient query budget nor construct an MPO from each query realization.

\subsection{Matched-budget estimation and learned-model accuracy}
\label{sec:exp_matched}
\label{sec:app_matched}

\paragraph{Question and common query budget.}
We now construct a model from each estimated Hamiltonian and compare global,
independent-local, and shared-local estimation under matched scalar-score
budgets.  All methods use the same full-support proposal and estimate the
same full-space operator, so domain truncation does not confound the
comparison.  The shared-local variant reuses the global samples to separate
the effect of local integration from the effect of sample reuse.

Let $\pi(x)=\prod_{j=1}^{D}\mathcal N(x_j;0,2)$ and $\pi_{C_j}$ be its
marginal on $C_j$, so the local proposal is $\pi_j=\pi_{C_j}$.
The full-space score contribution is
\[
 G_j(x):=(2\partial_j\Phi(x)-s_j(x)\Phi(x))
         (2\partial_j\Phi(x)-s_j(x)\Phi(x))^\top,
 \qquad H=\int\sum_jG_j(x)\,dx.
\]
As in the main text, $h_j(x_{C_j})$ is the local Gram integrand obtained
before identity embedding.  With $m=B/D$, define
\begin{equation}
\begin{aligned}
 \bar H^{\rm global}
   &:=\frac1m\sum_{i=1}^{m}
       \frac{\sum_{j=1}^{D}G_j(X_i)}{\pi(X_i)},
       && X_i\stackrel{\rm iid}{\sim}\pi,\\
 \bar A_j^{\rm shared}
   &:=\frac1m\sum_{i=1}^{m}
       \frac{h_j(X_{i,C_j})}{\pi_{C_j}(X_{i,C_j})},\\
 \bar A_j^{\rm independent}
   &:=\frac1m\sum_{i=1}^{m}
       \frac{h_j(Y_{j,i})}{\pi_{C_j}(Y_{j,i})},
       &&Y_{j,i}\stackrel{\rm iid}{\sim}\pi_{C_j}.
\end{aligned}
\label{eq:experimental_matched_estimators}
\end{equation}
For $v\in\{\mathrm{shared},\mathrm{independent}\}$, assemble
\[
 \bar H^v:=\sum_{j=1}^{D}I_K^{\otimes(\ell_j-1)}
                 \otimes\bar A_j^v\otimes I_K^{\otimes(D-u_j)}.
\]
All three estimators are unbiased for $H$.
The first two share the same full-vector score evaluations; the third uses
independent samples across terms.  Each uses exactly $B$ scalar score
components.  This is a query accounting convention, not a claim of equal
wall-clock oracle cost.  Shared-local terms are correlated across $j$ and do
not instantiate the independent-term assumption of
Lemma~\ref{thm:concentration}.

\paragraph{Model construction and evaluation.}
For each target, we use $K=4$, $D\in\{2,4\}$, total budgets
$B\in\{256,1024,4096,16384\}$, and ten seeds (43--52).
Global and shared-local estimates are paired within each trial.
We evaluate the empirical ground amplitude and its $\chi=4$ TT-SVD
compression, which coincide at $D=2$. The ground eigensolve is dense,
with maximum
ambient dimension $4^4=256$.  The population eigenvectors are never used to
initialize, select, or compress the empirical states.

\paragraph{Numerical checks.}
The independent full-Gram and local-embedding population constructions agree
to maximum entrywise difference $1.94\times10^{-12}$ over both targets and
dimensions.  The Hermite orthonormality error is $2.22\times10^{-16}$.
The $D=2$ global and shared-local estimates agree, as every local support is
the full coordinate set. Population energy and trace-norm error are
computed against the same
deterministic reference for all methods.
Reverse KL uses full-space Gauss--Hermite integration at orders 24 and 32,
with further orders $48,64,96,128,192$ as needed until the successive
difference is at most $\max\{10^{-8},10^{-3}|D_{\rm KL}|\}$.
Every evaluated model meets this convergence criterion; no seed is excluded.
The maximum normalization error is $3.78\times10^{-15}$, empirical eigenpair
residual is $2.55\times10^{-12}$, and independently integrated
Fisher--energy discrepancy is $7.39\times10^{-13}$. The computed KL values
satisfy the energy upper bound in~\eqref{eq:experimental_kl_certificate}.
As with the population study, quadrature convergence is a numerical check,
not a rigorous integration-error certificate.

\begin{figure}[t]
\centering
\includegraphics[width=\linewidth]{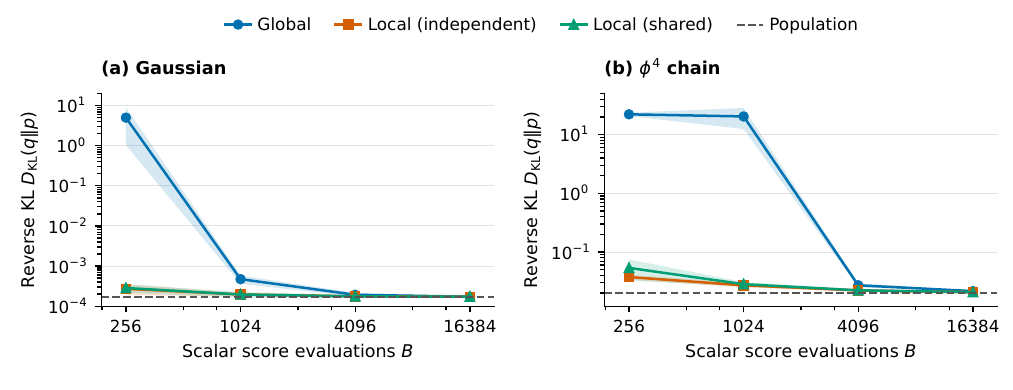}
\caption{Learned-model accuracy under matched scalar-score budgets at $D=4$,
$K=4$, and amplitude bond $\chi=4$ ($R\leq16$).
Curves and bands show means and sample standard deviations across ten seeds.
Dashed lines give the reverse KL of the compressed population ground states,
which are comparison values rather than lower bounds on reverse KL.
All estimators target the same full-space Hamiltonian; global and shared-local
estimation reuse the same full-score samples.
}
\label{fig:matched_estimator_learning}
\end{figure}

\paragraph{Results and interpretation.}
Figure~\ref{fig:matched_estimator_learning} and Table~\ref{tab:matched_full}
show that local estimation reaches population-level density accuracy with
fewer queries in these examples.
For the Gaussian at $B=1024$, reverse KL is
$(4.71\pm1.08)\times10^{-4}$ for global estimation,
$(1.96\pm0.20)\times10^{-4}$ for independent-local, and
$(1.95\pm0.31)\times10^{-4}$ for shared-local, compared with the compressed
population value $1.72\times10^{-4}$.
For $\phi^4$ at $B=4096$, the values are
$0.02717\pm0.00199$, $0.02216\pm0.00075$, and
$0.02233\pm0.00096$, respectively.
At $B=16384$, both local estimators are within $0.0005$ of the compressed
population reverse KL $0.020315$ in mean accuracy, while global estimation
reaches $0.02164\pm0.00055$.

Low-budget global fits are unstable: at $B=256$, the $\phi^4$ mean reverse
KL is $22.19\pm1.44$, compared with $0.03755\pm0.00382$ for
independent-local estimation.  The corresponding global empirical
low-energy spectrum is nearly degenerate.
These observations show sensitivity to estimation noise in the tested
small-sample regime, not an exponential query lower bound.
Even at the largest budget, mean relative operator error of the global
$\phi^4$ estimate, $\|\bar H-H\|_{S^\infty}/\|H\|_{S^\infty}$, is $3.56$
despite its much more accurate density.
Operator estimation and learned-density accuracy therefore provide distinct
diagnostics.

\begin{table}[tbp]
\centering
\caption{Complete $D=4$, $K=4$, $\chi=4$ matched comparison. Reverse KL is mean $\pm$ sample standard deviation across ten seeds. Scalar budgets count all score components.}
\label{tab:matched_full}
\begin{tabular}{llrrr}
\toprule
Target & $B$ & Global & Independent local & Shared local\\
\midrule
Gaussian & 256 & $5.02895\pm3.99251$ & $0.000270\pm0.000083$ & $0.000284\pm0.000068$\\
Gaussian & 1024 & $0.000471\pm0.000108$ & $0.000196\pm0.000020$ & $0.000195\pm0.000031$\\
Gaussian & 4096 & $0.000192\pm0.000020$ & $0.000176\pm0.000008$ & $0.000175\pm0.000013$\\
Gaussian & 16384 & $0.000172\pm0.000006$ & $0.000172\pm0.000006$ & $0.000172\pm0.000005$\\
$\phi^4$ & 256 & $22.18553\pm1.43613$ & $0.03755\pm0.00382$ & $0.05379\pm0.02113$\\
$\phi^4$ & 1024 & $20.42681\pm8.00331$ & $0.02717\pm0.00232$ & $0.02839\pm0.00245$\\
$\phi^4$ & 4096 & $0.02717\pm0.00199$ & $0.02216\pm0.00075$ & $0.02233\pm0.00096$\\
$\phi^4$ & 16384 & $0.02164\pm0.00055$ & $0.02079\pm0.00012$ & $0.02075\pm0.00023$\\
\bottomrule
\end{tabular}
\end{table}

The full-support proposals differ from the uniform compact-support
sampling assumed in Lemma~\ref{thm:concentration}. This comparison thus
measures the empirical effect of locality at equal query budgets rather
than testing that lemma's sufficient sample size.

\subsection{Implementation and scope}

The experiments were run on a MacBook Pro equipped with an Apple M3 Pro chip,
using double precision and CPU linear algebra with one BLAS thread.
The feature/bond sweep and matched-estimator comparison use Python 3.13.0,
NumPy 2.4.6, and SciPy 1.18.0; auxiliary score and tensor operations use
PyTorch 2.14.0 and opt\_einsum 3.4.0. Matplotlib 3.11.1 produces the vector
figures. The grids, quadrature rules, query counts, and random seeds are
specified with each experiment.

The three studies isolate feature approximation, bond compression, and
finite-query estimation. Their interpretation depends on keeping these
sources of error separate: increasing the bond leaves the irreducible term
unchanged at fixed features, while increasing the query budget reduces
sampling error without eliminating domain-truncation bias. The matched
comparison avoids that bias by estimating the same full-space Hamiltonian
for all methods. These controlled constructions support the theoretical
mechanisms; a scalable algorithm that preserves the guarantees remains a
separate question.
\FloatBarrier

\end{document}